\documentclass{article}

\usepackage{times}
\usepackage{graphicx} 
\usepackage{subfigure} 

\usepackage{natbib}

\usepackage{algorithm}
\usepackage{algorithmic}

\usepackage{hyperref}

\usepackage[accepted]{icml2017}

\icmltitlerunning{Generalized RBF kernel for incomplete data}

\usepackage{amssymb,amsmath,amsthm,amsfonts,dsfont}
\usepackage{verbatim}
\usepackage{stmaryrd}
\usepackage{caption}
\usepackage{color}
\usepackage{tikz}
\usepackage{multirow}
\usepackage{float}
\usetikzlibrary{matrix}
\usetikzlibrary{arrows}
\usetikzlibrary{decorations.pathreplacing}

\newtheorem{theorem}{Theorem}[section]

\newtheorem{observation}{Observation}[section]

\theoremstyle{remark}

\def\const{\mathrm{const}}

\def\R{\mathbb{R}}

\def\1{\mathds{1}}

\def\X{\!\!\!\!\!\!\! & }

\def\span{\mathrm{span}}

\def\det{\mathrm{det}}
\def\L2{L_2}
\def\our{\mbox{$\bf genRBF$}}

\newcommand\il[1]{\langle #1 \rangle}

\begin{document} 

\twocolumn[
\icmltitle{Generalized RBF kernel for incomplete data}




\begin{icmlauthorlist}
\icmlauthor{Marek \'Smieja}{goo}
\icmlauthor{Lukasz Struski}{goo}
\icmlauthor{Jacek Tabor}{goo}
\end{icmlauthorlist}

\icmlaffiliation{goo}{Jagiellonian University, Krak\'ow, Poland}

\icmlcorrespondingauthor{Marek \'Smieja}{marek.smieja@ii.uj.edu.pl}

\icmlkeywords{incomplete data, missing data, Gaussian kernel, RBF, SVM}

\vskip 0.3in
]



\printAffiliationsAndNotice{}  

\begin{abstract} 
We construct \our{} kernel, which generalizes the classical Gaussian RBF kernel to the case of incomplete data. We model the uncertainty contained in missing attributes making use of data distribution and associate every point with a  conditional probability density function. This allows to embed incomplete data into the function space and to define a kernel between two missing data points based on scalar product in $\L2$. Experiments show that introduced kernel applied to SVM classifier gives better results than other state-of-the-art methods, especially in the case when large number of features is missing. Moreover, it is easy to implement and can be used together with any kernel approaches with no additional modifications.
\end{abstract}

\section{Introduction} \label{se:introduction}


Incomplete data analysis is an important part of data engineering and machine learning, since it appears naturally in many practical problems. In particular, in medical diagnosis, a doctor may be unable to complete the patient examination due to the deterioration of health status or lack of patient's compliance \cite{burke1997compliance}; in object detection, the system has to recognize partially hidden faces \cite{mahbub2016partial} or identify shapes from corrupted images \cite{berg2005shape}; in chemistry, the complete analysis of compounds requires high financial costs \cite{stahura2004virtual}. In consequence, the understanding and the appropriate representation of such data is of great practical  importance.

The choice of the method for analyzing incomplete data depends on the reasons why data are missing \cite{schafer1997analysis}. If missing entires are generated completely randomly (MCAR) or at least do not depend on missing values (MAR), then one can reliably estimate the probability distribution of incomplete data by the mixture model applying EM algorithm. Otherwise, we can model the missing data mechanism, which however leads to more complex solutions, or we can completely discard missing features, which drastically reduces available information (NMAR).

In this paper, we propose \our{}, a generalization of RBF (radial basis function) kernel to the case of incomplete data. To briefly explain our approach, let us recall that classical RBF can be constructed by embedding every point into the function space (with regularization by Gaussian kernel) and applying a standard scalar product in $\L2$. To generalize this process for incomplete data, we model the uncertainty on missing coordinates by restricting data density to absent attributes (for a simplicity, we use a single Gaussian as a density model). In consequence, a 
missing data point is represented as a regularization of a singular Gaussian density on a respective affine subspace of the data.
The illustration of the above process is presented in Figure \ref{fig:representation}.

\begin{figure*}[t]
	\centering
	\subfigure[Complete data point $x = (-1,-2)^T$ and missing data point $y = (?, 1)^T$ identified with an affine subspace.]{\includegraphics[height=2in]{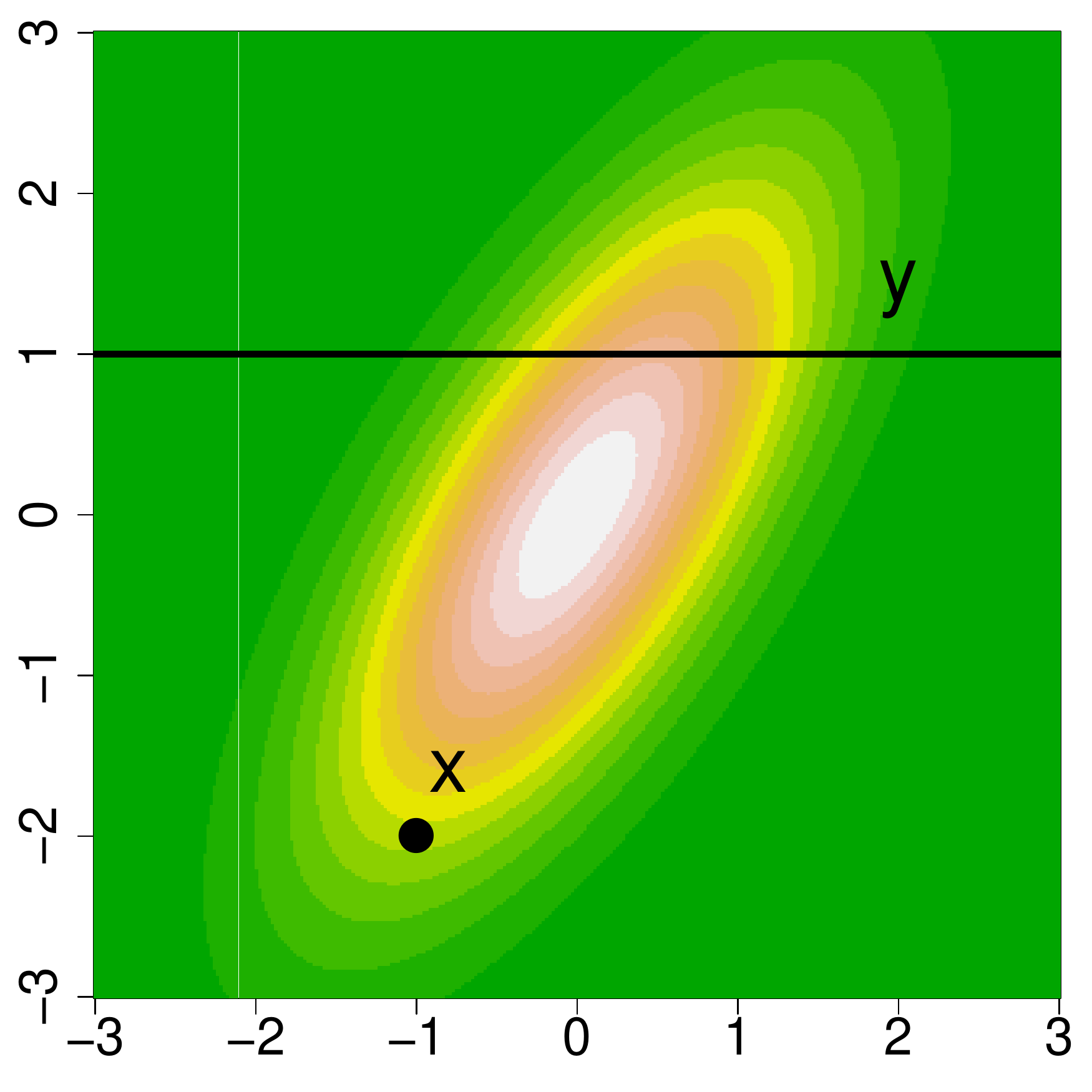}  \label{c1}} \quad
	\subfigure[Missing data point represented by degenerate density, which has a support restricted to missing attributes.]{\includegraphics[height=2in]{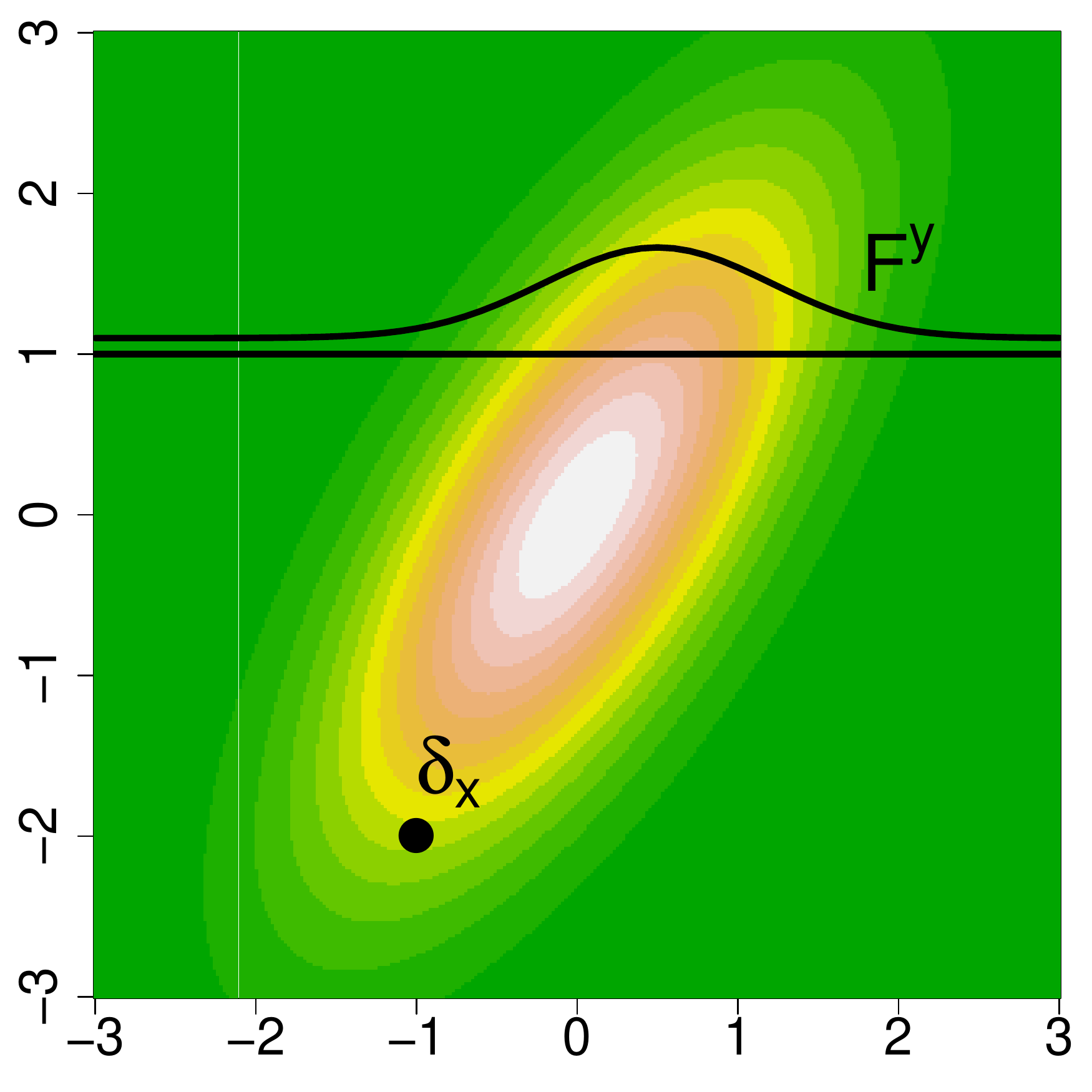}  \label{c2}} \quad
	\subfigure[Embedding of $x$ and $y$ into $\L2$ using the convolution (regularization) operator.]{\includegraphics[height=2in]{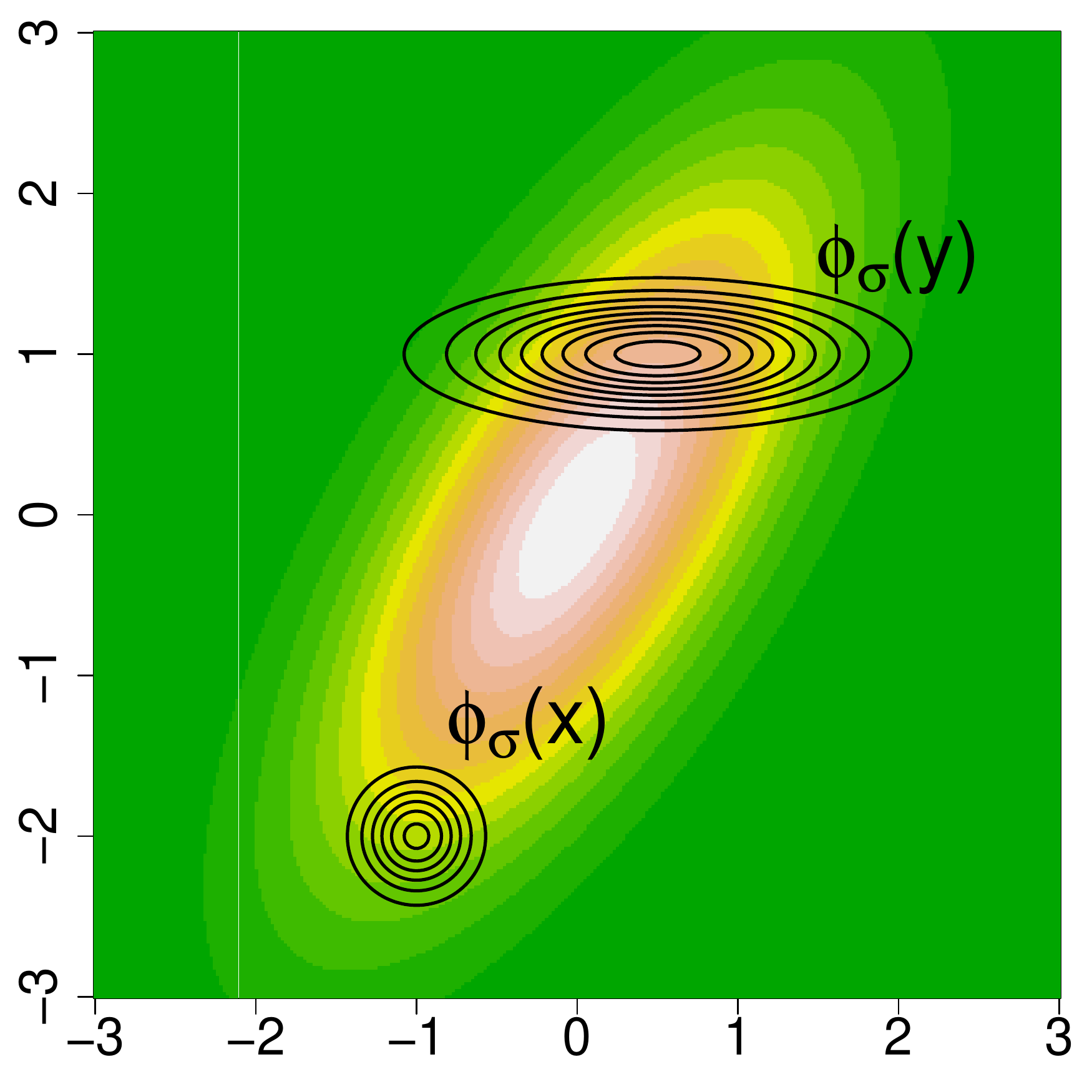} \label{c3}}
	\caption{Representation and embedding of complete and missing data points into $\L2$ space making use of Gaussian density estimated from data.}
	\label{fig:representation}
\end{figure*}

Main features of \our{} can be summarized as follows:
\begin{itemize}
\item  \our{} is easy to implement and can be used together with any kernel approach,
\item it does not perform any direct imputations, 
\item \our{} is effective, resistant to possible perturbations and robust to the number missing entries,
\item SVM classifier which uses \our{} obtains better results than existing 
state-of-the-art methods. 
\end{itemize}

\section{Related work}

The most common approach to learning from incomplete data is known as deterministic imputation \cite{mcknight2007missing}. In this two-step procedure, the missing features are filled first, and only then a standard classifier is applied to the complete data \cite{little2014statistical}. Although the imputation-based techniques are easy to use for practitioners, they lead to the loss of information which features were missing and do not take into account the reasons of missingness. To preserve the information of missing attributes, one can use an additional vector of binary flags, indicating which coordinates were missing.

The second popular group of methods aims at building a probabilistic model of incomplete data. If data are missing at random, then it is possible to apply EM algorithm to estimate a density of data by the mixture of parametric models \cite{ghahramani1994supervised, schafer1997analysis}. Consequently, this allows to generate the most probable values from obtained probability distribution for missing attributes (random imputation) \cite{mcknight2007missing} or to learn a decision function directly based on the distributional model. The second option was already investigated in the case of logistic regression \cite{williams2005incomplete}, kernel methods \cite{smola2005kernel, williams2005analytical} or by using second order cone programming \cite{shivaswamy2006second}. One can also estimate the parameters of the probability model and the classifier jointly, which was considered in \cite{dick2008learning, liao2007quadratically}. As it was mentioned above, the limitation of these techniques is the assumption about the process of missing data generation. If missing data depends on the unobserved features then there is no guarantee to get a reasonable estimation of data density.

There is also a group of methods which do not make any assumptions about the missing data mechanism and make a prediction from incomplete data directly. In  \cite{chechik2008max} a modified SVM classifier is trained by scaling the margin according to observed features only. The alternative approaches to learning a linear classifier, which avoid features deletion or imputation, are presented in \cite{dekel2010learning, globerson2006nightmare}. In \cite{grangier2010feature} the embedding mapping of feature-value pairs is constructed together with a classification objective function. Finally, the authors of \cite{hazan2015classification} design an algorithm for kernel classification that performs comparably to the classifier which have an access to complete data, under low-rank assumption (every vector can be reconstructed from the observed attributes).

\section{Generalized RBF}

In this section we present the construction of \our{} kernel function. We begin with the description of incomplete data by affine subspaces. Next, we show how to use the information contained in data distribution to model the uncertainty on missing coordinates and, in consequence, how to represent incomplete data points by probability measures. This identification allows to apply the reasoning commonly used in classical RBF kernels and to define an analogue formula for a kernel function in the case of incomplete data.
The visual representation of the idea behind \our{} kernel is given in Figure \ref{fig:kernel}.

\begin{figure*}[t]
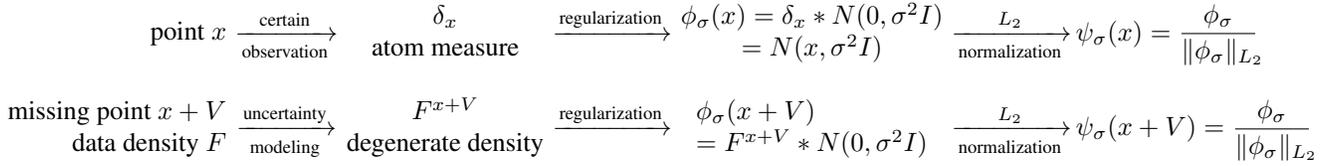

	\centering
$$
\!\!\!\!
\begin{array}{rcccccl}
 \begin{array}{c} \text{point }x \end{array} \X \xrightarrow[\text{observation}]{\text{certain}} \X \begin{array}{c} \delta_x \\ \text{atom measure} \end{array} \X \xrightarrow[]{\text{regularization}} \X \begin{array}{c} \phi_{\sigma}(x) = \delta_x * N(0, \sigma^2 I) \\ = N(x, \sigma^2 I)\end{array} 
 \X \xrightarrow[\text{normalization}]{\L2} 
 \X \,\, \psi_\sigma(x)  = \displaystyle{\frac{\phi_\sigma}{\|\phi_\sigma\|_{\L2}\!\!}}
\\[4.0ex]
 \begin{array}{r}\text{\!\!\!\!\!missing point } x+V  \\ \text{data density }F\end{array} \X \xrightarrow[\text{modeling}]{\text{uncertainty}} \X \begin{array}{c} F^{x+V} \\ \text{degenerate density} \end{array} \X \xrightarrow[]{\text{regularization}} \X \begin{array}{l} \phi_{\sigma}(x+V) \\ = F^{x+V} * N(0, \sigma^2 I)
\end{array} 
 \X \xrightarrow[\text{normalization}]{\L2} 
 \X \,\, \psi_\sigma(x+V) = \displaystyle{\frac{\phi_\sigma}{\|\phi_\sigma\|_{\L2}\!\!}}
\end{array}
$$
	\caption{Diagram showing the comparison of constructions of classical RBF kernel and \our{}. Observe that the last step ($\L2$ normalization) is needed to ensure that the
	resulting kernel function $K_\sigma(x,y)=\langle \psi_\sigma(x),\psi_\sigma(y)\rangle_{\L2}$ satisfies the condition $K_\sigma(x,x)=1$.}
	\label{fig:kernel}
\end{figure*}

\subsection{Subspace representation}

An incomplete data point in $\R^N$ is typically understood as a pair $(x, J_x)$, where $x \in \R^N$ and $J_x \subset \{1,\ldots,N\}$ is a set of indices of missing atributes. We can associate a missing data point $(x,J_x)$ with an affine subspace $x + \span(e_j)_{j \in J_x}$, where $(e_j)_j$ is the canonical base of $\R^N$. Let us observe that $x+ \span(e_j)_{j \in J_x}$ is a set of all $N$-dimensional vectors, which coincide with $x$ on the coordinates different from $J_x$. 

In practice, data are often transformed by linear mappings (e.g. whitening) in a preprocessing stage. For this purpose, we generalize the above representation to arbitrary affine subspaces, which do not have to be generated over canonical bases. Therefore, we assume that the set of incomplete data consists of affine subspaces of the form $ x + V$, where $x \in \R^N$ is arbitrary and $V$ is a linear subspace of $\R^N$.

If $f:\R^N \ni w \to Aw+b$ is a an affine map, then we can transform a missing data point $x+V$ into another missing data point by the formula:
$$
f(x+V)=\{Aw+b:w \in x+V\}.
$$
The linear part of $f(x+V)$ is given by 
$$
f(x+V)-f(x)=AV.
$$
One can easily compute and represent $AV$ if the orthonormal base $v_1,\ldots,v_n$ of $V$ is given, we simply orthonormalize the sequence $Av_1,\ldots,Av_n$. 

For example, given the whitening operator:
$$
\mathrm{Whitening}(x) = \Sigma^{-1/2}(x-m),
$$
where $\Sigma$ denotes the covariance and $m$ is the mean vector, a missing data point $x+V$ is transformed to
$$
\mathrm{Whitening}(x+V) = \Sigma^{-1/2}(x-m) + \Sigma^{-1/2} V.
$$
The linear part $V$ is mapped to $\Sigma^{-1/2} V$, which has to be shifted by a vector $\Sigma^{-1/2}(x-m)$ (classical whitening operator applied to $x$).

\subsection{Missing data point as a probability density}

Subspace representation of incomplete data gives no information where the point is localized on the affine subspace. To add this information and to reduce the uncertainty connected with missing attributes we need the knowledge of the distribution of data. It allows to identify a missing data point with a degenerate density with support restricted to the affine subspace.

To realize this goal in practice, we need to perform a density estimation on the incomplete data set. It is well-known that it is possible to apply the EM algorithm to obtain the estimation by a mixture of parametric models if data satisfy missing at random assumption (MAR). Although in more general case the calculated density might be unreliable, we only use it to reduce the uncertainty on absent attributes (for observable attributes this density is not used). Therefore, we assume that some estimation $F$ of data density is given. 

A complete data point $x$ (with no missing coordinates) can be identified with atom Dirac measure $\delta_x$ (a measure that takes value $x$ with probability $1$), because there is no uncertainty connected with this example. If we have an incomplete data point $x+V$ then the uncertainty is connected with its missing part, which can be modeled by a restriction  of density $F$ to the affine subspace $x+V$, which is denoted by $F_{x+V}$ (conditional density). Let us recall that if $v = [v_1,\ldots,v_n]$ is an orthonormal base of $V$ then
$$
F_{x+V}(x + v \alpha) = \frac{F(x+ v \alpha)}{\int_{x+V} F(y)dy } \text{ for } \alpha \in \R^n.
$$
This conditional density is defined for points contained in a subspace $x+V$. Since we work in $N$ dimensional space, it is convenient to extend this conditional density to the degenerate density in original $\R^N$ space. In other words, we have to form a density $F^{x+V}$ from the conditional density $F_{x+V}$ by:
$$
F^{x+V}(y) = \left\{\begin{array}{ll}
0, & y \notin x+V,\\
F_{x+V}(y), & y \in x+V,
\end{array}
\right. \text{ for } y \in \R^N.
$$
We use a density $F^{x+V}$ to represent missing data points in $\R^N$.

For a simplicity and clarity of presentation, we restrict our attention to the case of Gaussian densities and assume  that $F=N(m,\Sigma)$, given by
$$
N(m,\Sigma)(x)=\frac{1}{(2\pi)^{N/2} \det^{1/2} \Sigma}\exp(-\tfrac{1}{2}\|x-m\|^2_\Sigma),
$$
is a Gaussian estimation of the distribution on incomplete data set, where $\|y\|^2_\Sigma = y^T\Sigma^{-1}y$ denotes the square of Mahalanobis norm of $y \in \R^N$. Although a single Gaussian density might not be enough to fully reflect a complex structure of data, it is robust to the number of missing attributes, can be easily computed in practice and does not require so strong assumptions on missing data mechanism. 

Below, we present how to obtain the conditional density $F_{x+V} = N(m_V, \Sigma_V)$ and corresponding density $F^{x+V} = N(m^V, \Sigma^V)$ in the original space from a data space distribution $F = N(m,\Sigma)$. 
\begin{observation} \label{obs:1}
	Let $F_{x+V} = N(m_V,\Sigma_V)$ be a density in the affine subspace $x+V$. If  $v=[v_1,\ldots,v_n]$ is an orthonormal base of $V$ then the corresponding density in the original $\R^N$ space equals $F^{x+V} = N(m^V, \Sigma^V)$, where
	$$
	m^V = x+vm_V \text{ and } \Sigma^V = v\Sigma_V v^T.
	$$
\end{observation}

\begin{proof}
	If a random vector $Y$ has a mean $m_Y$ and a covariance $\Sigma_Y$, then $\Phi(Y) = A Y+b$ has the mean $A m_Y+b$ and the covariance $A\Sigma_Y A^T$. We apply this fact to the map $F^{x+V}:\R^n \ni \alpha=[\alpha_1,\ldots,\alpha_n]^T \to v\alpha+x \in \R^N$, which completes the proof.
\end{proof}
Observe that $F^{x+V} = N(m^V, \Sigma^V)$ given above is a degenerate density in $\R^N$ iff $n < N$, i.e. the covariance matrix $\Sigma^V$ is singular (invertible).

Now, we discuss the inverse problem:
\begin{observation} \label{obs:2}
	Let $F=N(m,\Sigma)$ be a normal density in $\R^N$. We assume that $x+V$ is an affine subspace of $\R^N$ and $v=[v_1,\ldots,v_n]$ is an orthonormal base of $V$. Then, the conditional density $F_{x+V}$ in the space $x+V$ in the base given by $v$ equals $N(m_V,\Sigma_V)$, where
	$$
	\Sigma_V=(v^T\Sigma^{-1}v)^{-1} \text{ and }m_V=\Sigma_V[v^T\Sigma^{-1}(m-x)].
	$$
\end{observation}

\begin{proof}
	Let us recall that the formula of normal density can be written as:
	$$
	w \to Z \cdot \exp(-\tfrac{1}{2}(w-m)^T\Sigma^{-1}(w-m)),
	$$
	where $Z$ is a normalization factor. Now, restricting the quadratic function $w \to (w-m)^T\Sigma^{-1}(w-m)$ to the space $x+V$ by putting $w=x+v\alpha$ we get
	$$
	\begin{array}{l}
	\alpha \to (x+v\alpha-m)^T\Sigma^{-1}(x+v\alpha-m) \\[1ex]
	=\alpha^T (v^T\Sigma^{-1}v)\alpha-2[v^T\Sigma^{-1}(m-x)]^T\alpha+const.
	\end{array}
	$$
	Finally, by the canonical form of the quadratic function\footnote{Recall the formula $\alpha^TA\alpha+b^T\alpha+c$, for symmetric $A$, can be rewritten as $(\alpha-\alpha_0)^TA(\alpha-\alpha_0)+\const$, for $\alpha_0=-\frac{1}{2}
		A^{-1}b$.} we get that this mapping equals
	$$
	\alpha \to (\alpha-m_V)^T\Sigma_V^{-1}(\alpha-m_V)+const,
	$$
	where
	$$
	\Sigma_V=(v^T\Sigma^{-1}v)^{-1} \text{ and }m_V=\Sigma_V[v^T\Sigma^{-1}(m-x)].
	$$
\end{proof}

Taking the above two observations together, we can calculate both densities from the original density $F$. In consequence, we represent a missing data point $x+V$ by a degenerate Gaussian density $N(m^V,\Sigma^V)$. As it was mentioned, this identification only influences absent attributes and has no effects on observable features.

\subsection{Kernel construction}

To define a scalar product (kernel function) on incomplete data, we will adapt the reasoning behind classical RBF kernels to the case of probabilistic representations introduced in previous subsection. Let us observe that the construction of classical RBF kernel for a complete data can be decomposed into the following steps:
\begin{itemize}
\item We map every point $x \in \R^N$ to Dirac measure $\delta_x$.
\item Next, we embed it into $\L2$ space by taking the convolution (regularization)
with $N(0,\sigma^2 I)$, where $\sigma$ is a fixed paramter: 
\begin{equation} \label{eq:classEmbed}
\phi_\sigma(x) = \delta_x * N(0,\sigma^2 I) = N(x, \sigma^2 I),
\end{equation}
\item Then, we apply the $\L2$ normalization
$$
\psi_\sigma(x)=\frac{\phi_\sigma(x)}{\|\phi_\sigma(x)\|_{\L2}\!\!\!}\,.
$$
\item Finally, we apply the scalar product in $\L2$ space between embeddings 
to define the kernel function $K_\sigma$
$$
K_\sigma(x, y) = \il{\psi_\sigma(x), \psi_\sigma(y)}_{\L2}.
$$
Due to the normalization, we have $K_\sigma(x,x) = 1$.
\end{itemize}
If not stated otherwise, $\| \cdot \|$ and $\il{\cdot, \cdot}$ will denote classical norm and scalar product in $\L2$ space, respectively.

To perform an analogue procedure in the case of missing data identified with Gaussian densities, let us introduce basic notations. We recall, that the standard scalar product in $\L2$ space is given by
$$
\il{F,G}=\int F(x)G(x)dx, \text{ for } F,G \in \L2.
$$
In the case of Gaussian densities, the above scalar product can be easily computed by \cite{petersen2008matrix}:
\begin{equation} \label{eq:scalarProduct}
	\il{N(m_1,\Sigma_1),N(m_2,\Sigma_2)}=N(m_1-m_2,\Sigma_1+\Sigma_2)(0),
\end{equation}
where $N(m_i,\Sigma_i)$ are non-degenerate Gaussians. 

We also need the notion of convolution, which for densities $F,G \in \L2$ is defined by
$$
(F * G)(y)=\int F(x-y) G(y) dx.
$$
If $F$ is a measure with a mean $m_F$ and a covariance $\Sigma_F$, then the convolution $F * N(0,\sigma^2 I)$, where $I$ is an identity matrix and $\sigma > 0$, is a measure with a mean $m_F$ and a covariance $\Sigma_F+\sigma^2 I$. The convolution of normal densities is a normal density and,
$$
N(m,\Sigma) * N(0,\sigma^2 I)=N(m,\Sigma+\sigma^2 I).
$$
The above formula also holds for degenerate normal densities. In consequence, this operator works as a regularization and allows to transform a degenerate density into a non-degenerate one.

Let us now calculate the normalized embedding of missing data point $x+V$ into $\L2$ space. For a fixed $\sigma >0$, we have 
\begin{equation} \label{eq:conv}
\begin{array}{ll}
	\phi_\sigma(x+V) & = N(m^V,\Sigma^V) *N(0,\sigma^2 I)\\[1ex]
	&=N(m^V,\Sigma^V+\sigma^2 I),
	\end{array}
\end{equation}
where $m^V, \Sigma^V$ follow from Observations \ref{obs:1} and \ref{obs:2}. Since we are interested in normalized embedding, we put:
$$
\psi_\sigma(x+V) = \frac{\phi_\sigma(x+V)}{\|\phi_\sigma(x+V)\|}.
$$

To define a kernel function on incomplete data, we simply calculate the scalar product in $\L2$ space between embeddings of missing data points. More precisely, for two missing data points $x+V$ and $y+W$, we put:
\begin{multline}\label{eq:scalarProb}
	K_\sigma(x+V,y+W) = \il{\psi_\sigma(x+V) ,\psi_\sigma(y+W) } \\
	 = \displaystyle \frac{\il{N(m^V,\Sigma^V + \sigma^2 I),N(m^W,\Sigma^W+ \sigma^2 I)}}{\|N(m^V,\Sigma^V + \sigma^2 I)\| \cdot \| N(m^W,\Sigma^W + \sigma^2 I)\|}.
\end{multline} 
where $\sigma > 0$ is fixed.

\begin{table}[t]
    \caption{Summary of data sets from UCI repository.}\label{tab:uci-data}
    \centering
    \footnotesize
    \begin{tabular}{l c c c }
        \hline
        \textbf{Data set} & \textbf{\#Instances} & \textbf{\#Attributes} & \textbf{Classes ratio}  \\ \hline 
        \textit{Australian} & 690 & 14 & 0.56\\ 
        \textit{Bank} & 1372 & 5 & 0.56 \\ 
        \textit{Breast cancer} & 699 & 8 & 0.66 \\ 
        \textit{Crashes} & 540 & 18 & 0.91  \\  
        \textit{Heart} & 270 & 13 & 0.56 \\ 
        \textit{Ionosphere} & 351 & 34 & 0.64 \\ 
        \textit{Liver disorders} & 345 & 7 & 0.58 \\ 
        \textit{Pima} & 768 & 8 & 0.65  \\ \hline
    \end{tabular}
\end{table}

\begin{figure*}[t]
	\centering
	\includegraphics[width=0.5\textwidth]{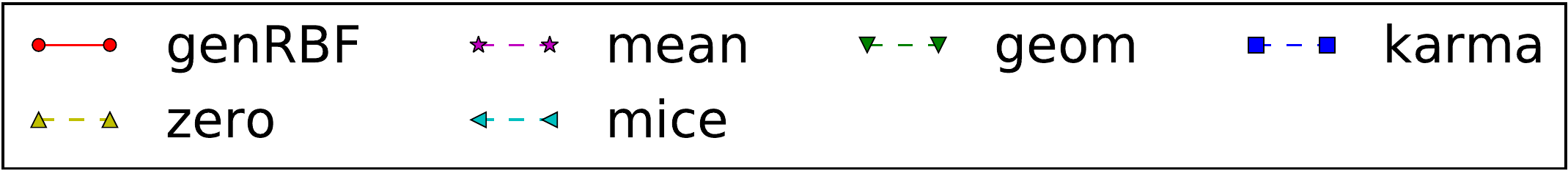}\\
	\includegraphics[width=0.24\textwidth]{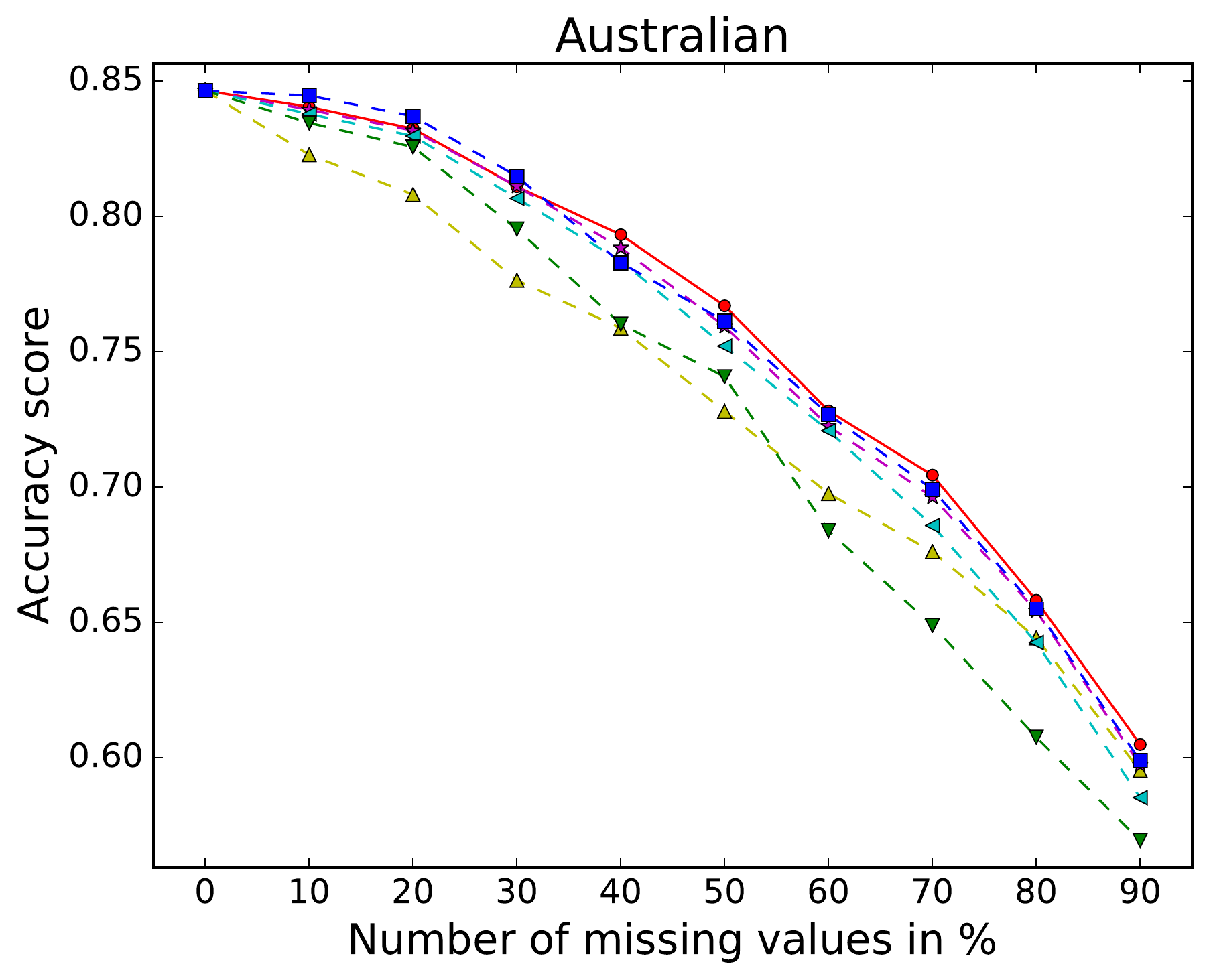}
	\includegraphics[width=0.24\textwidth]{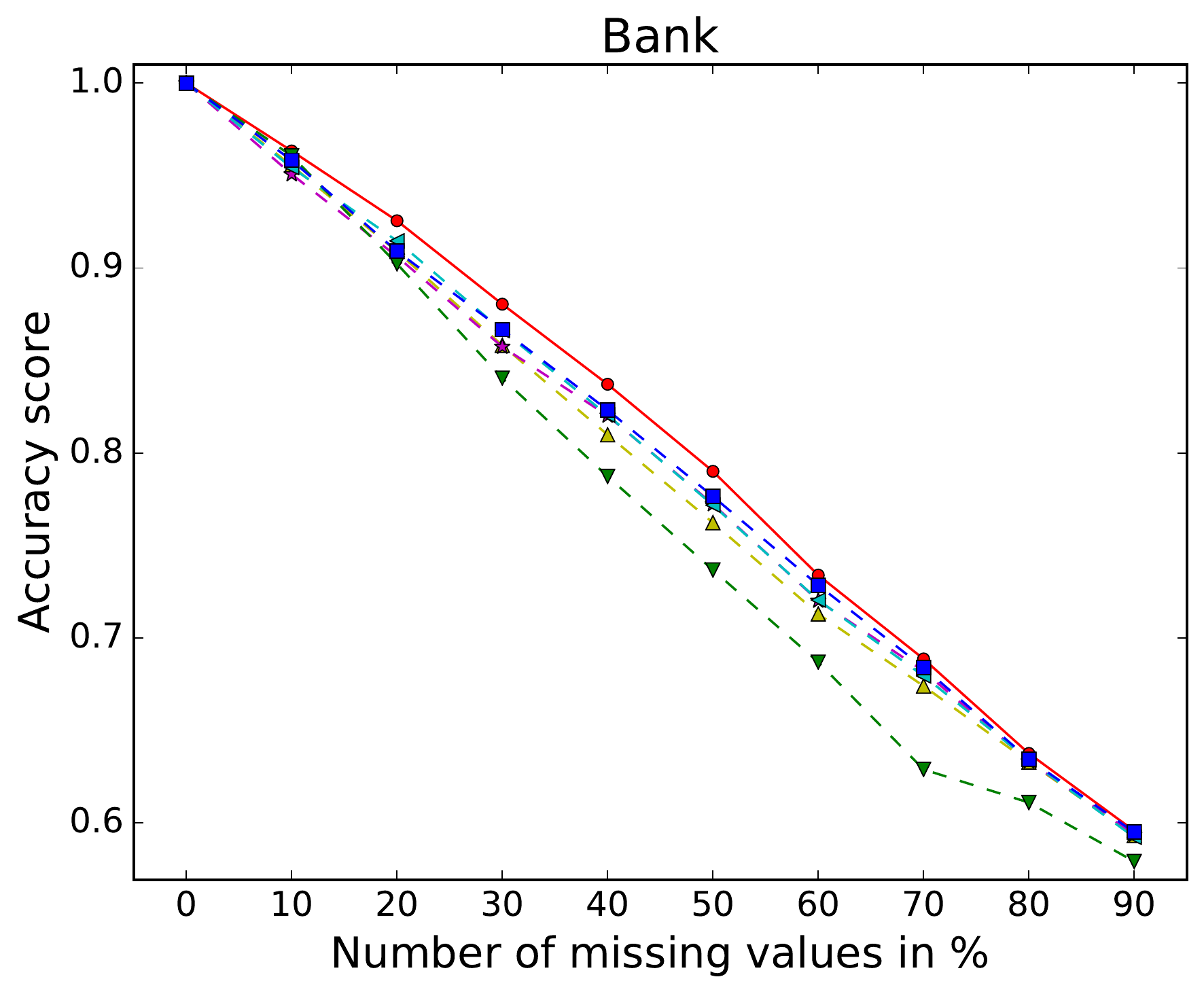}
	\includegraphics[width=0.24\textwidth]{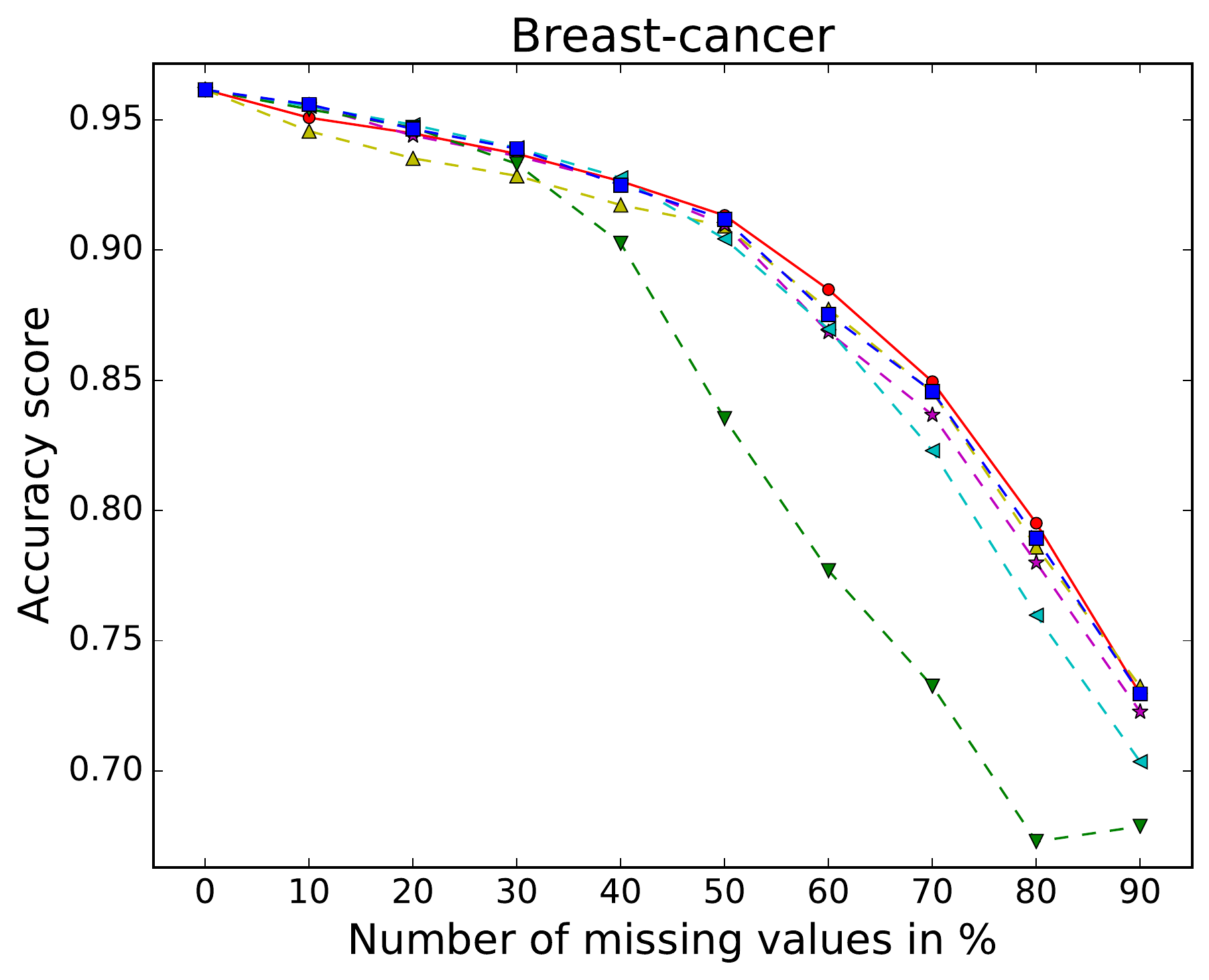}
	\includegraphics[width=0.24\textwidth]{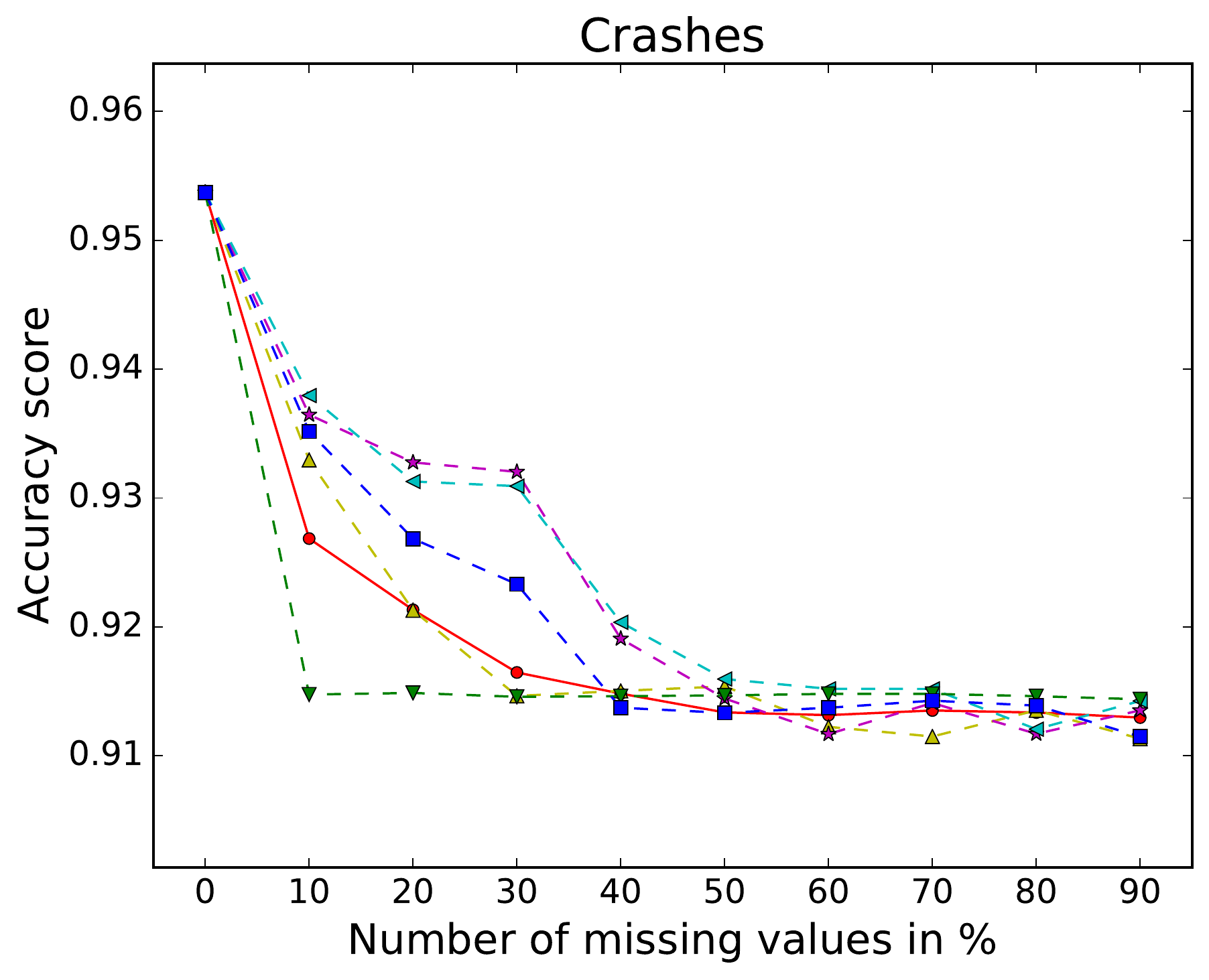}\\
	\includegraphics[width=0.24\textwidth]{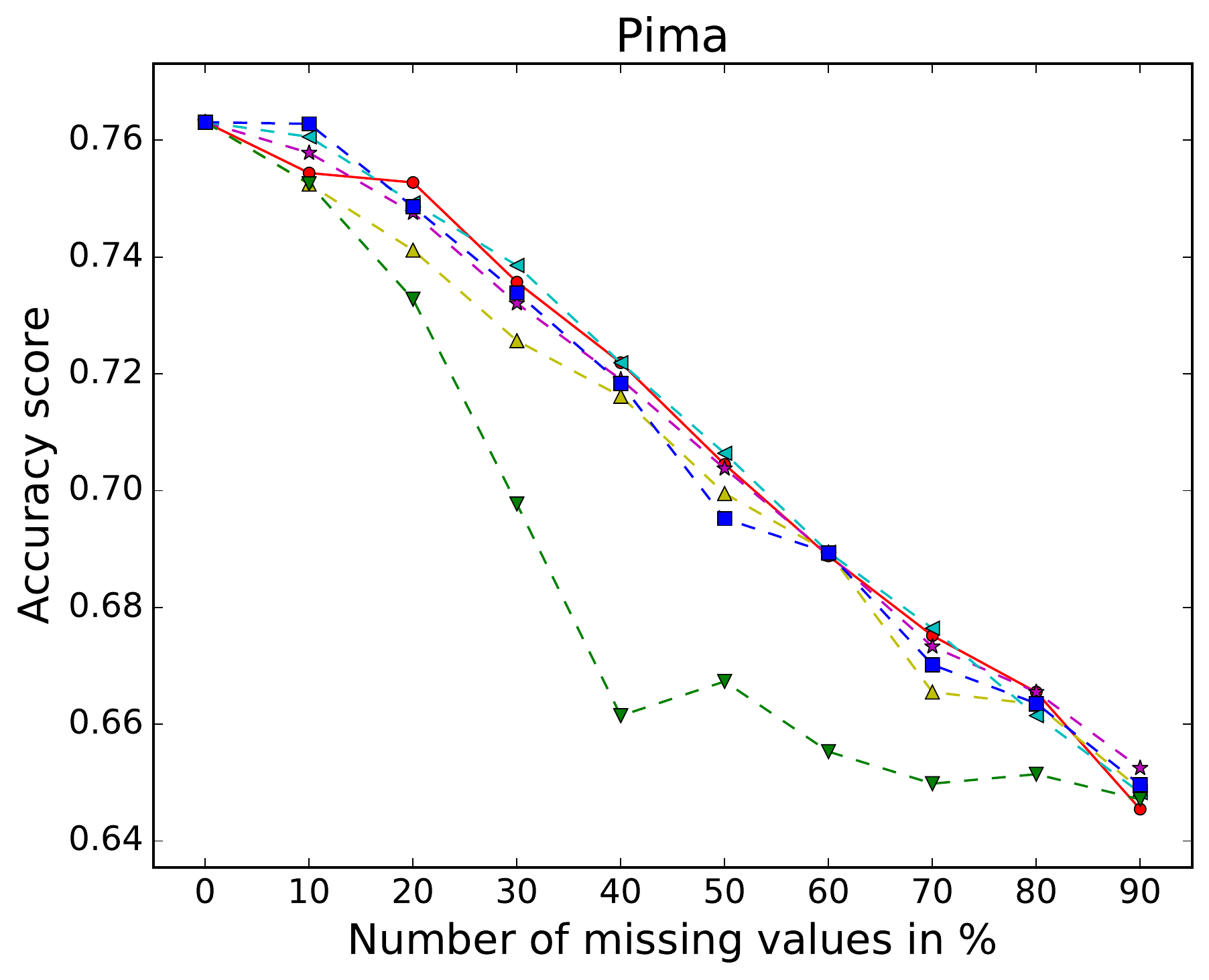}
	\includegraphics[width=0.24\textwidth]{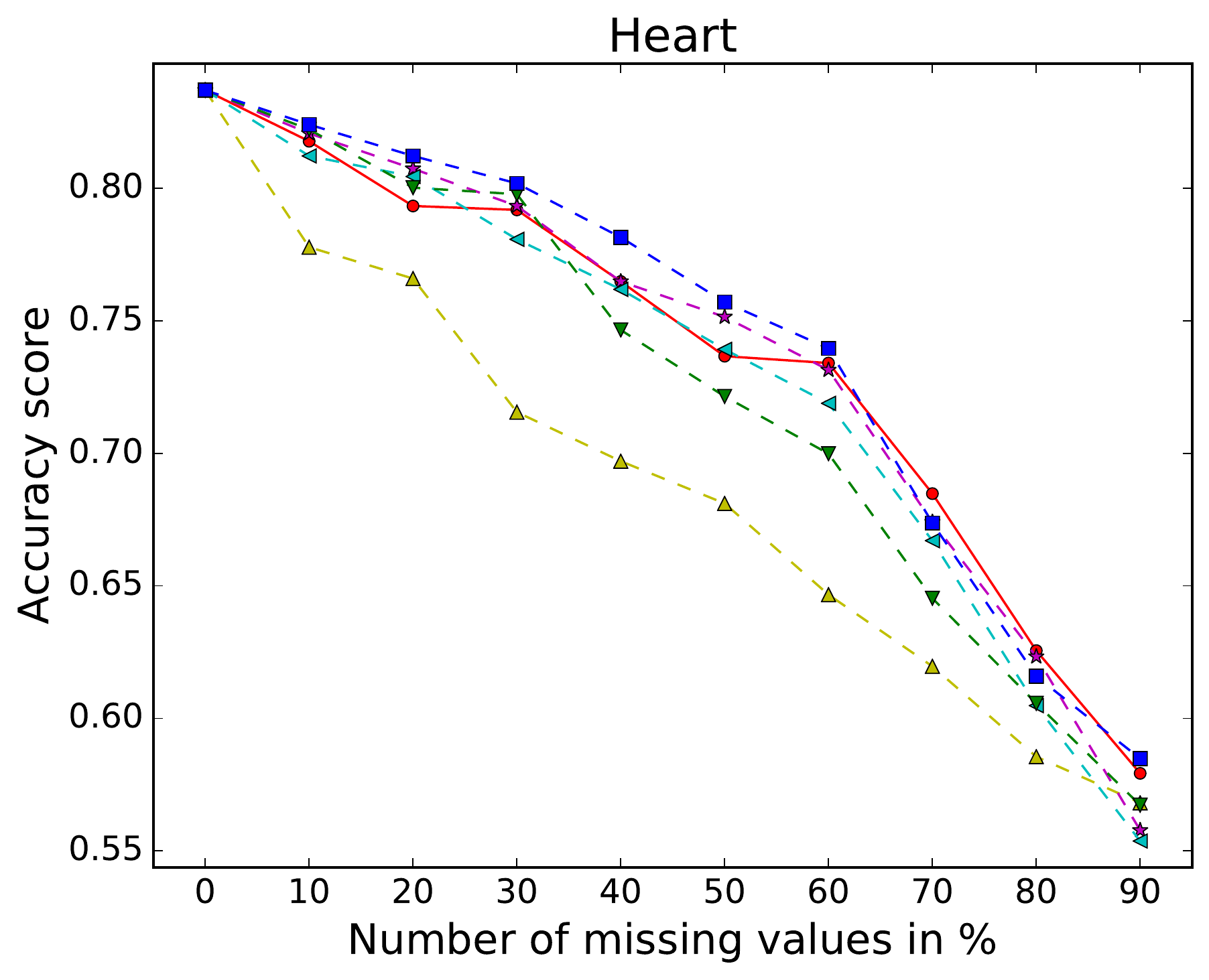}
	\includegraphics[width=0.24\textwidth]{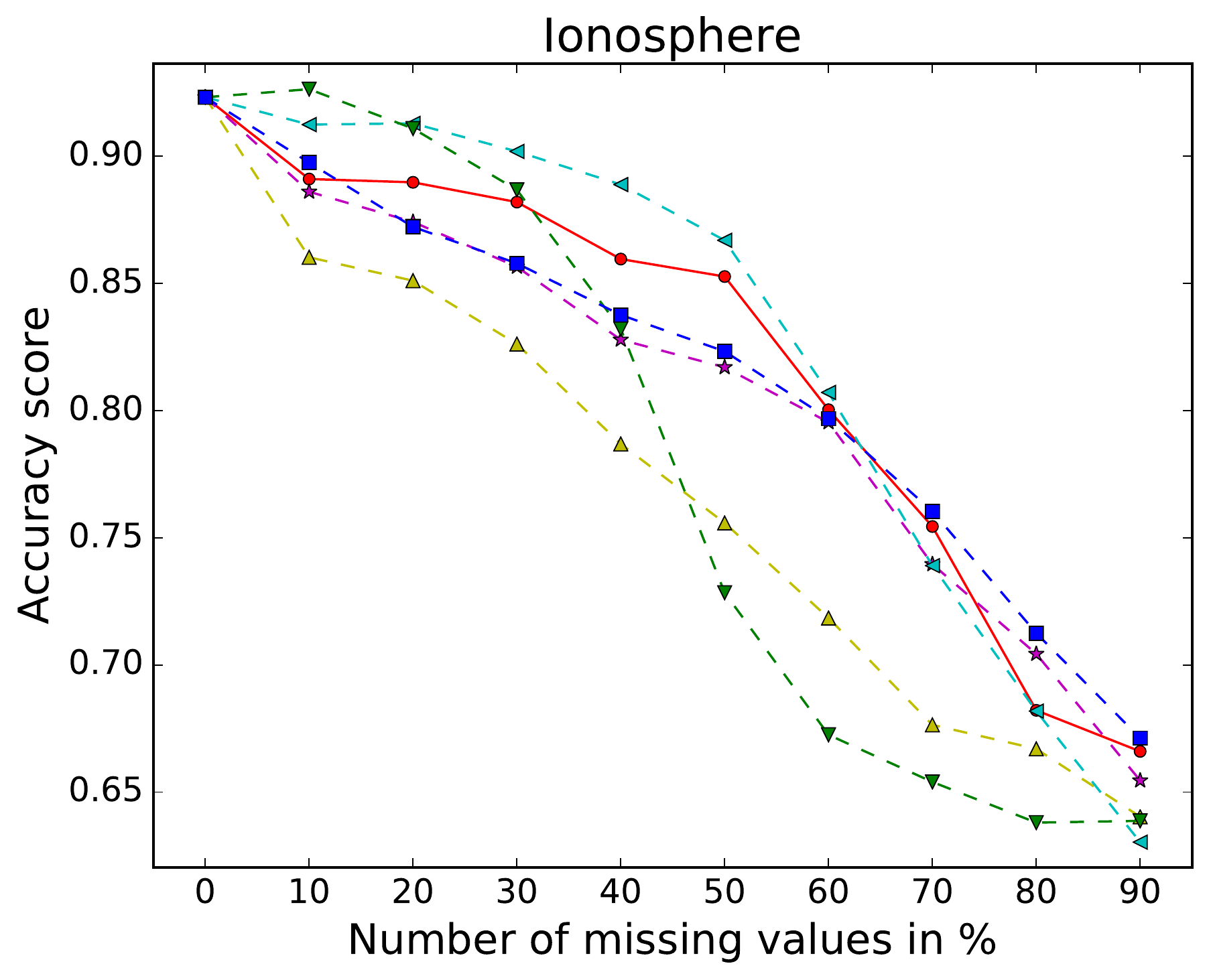}
	\includegraphics[width=0.24\textwidth]{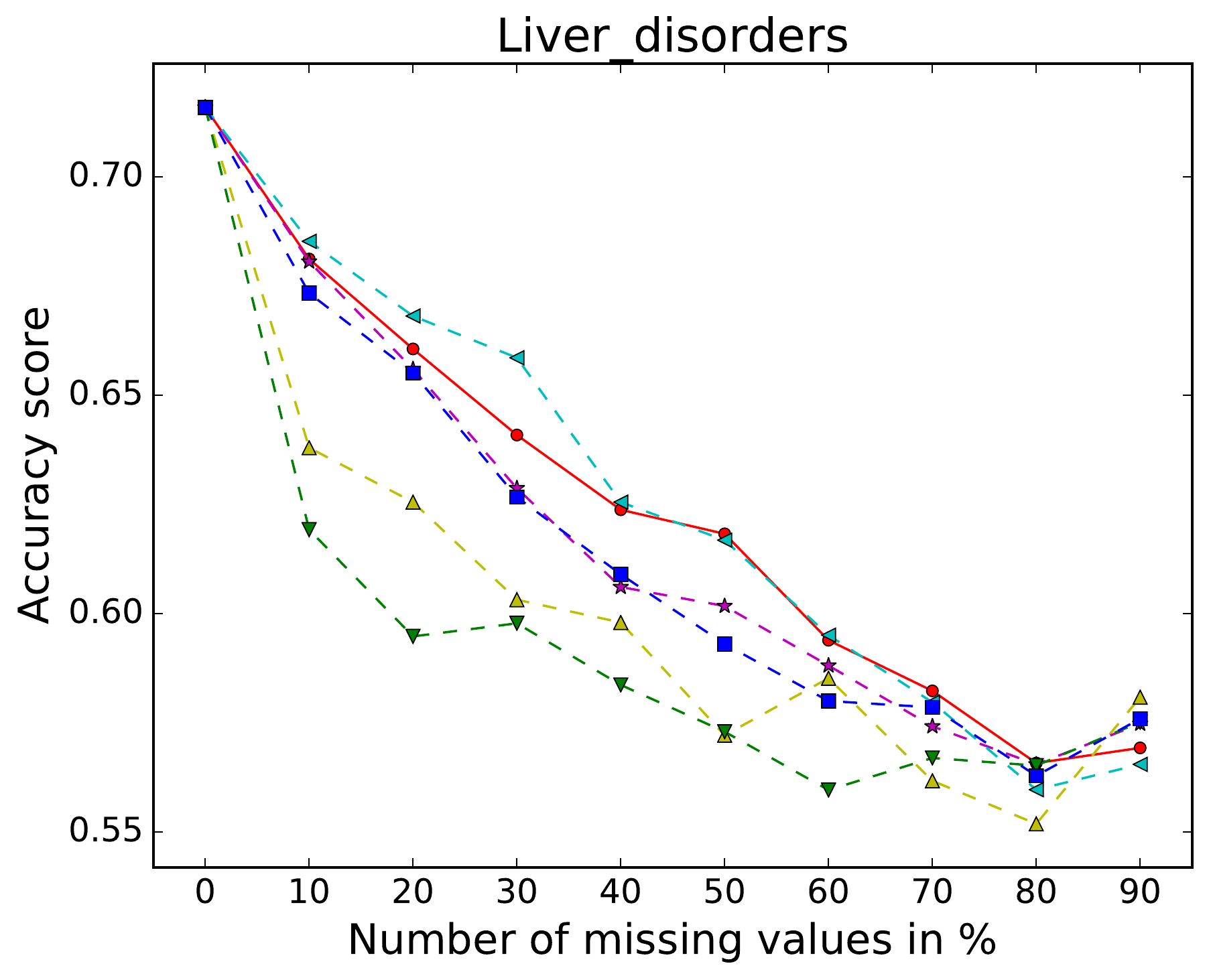}
	\caption{Classification results measured by accuracy reported on test set when missing entries satisfy MCAR.}
	\label{fig:uciMCAR}
\end{figure*}

\begin{figure*}[t]
	\centering
	\includegraphics[width=0.5\textwidth]{./figs/legend-crop.pdf}\\
	\includegraphics[width=0.24\textwidth]{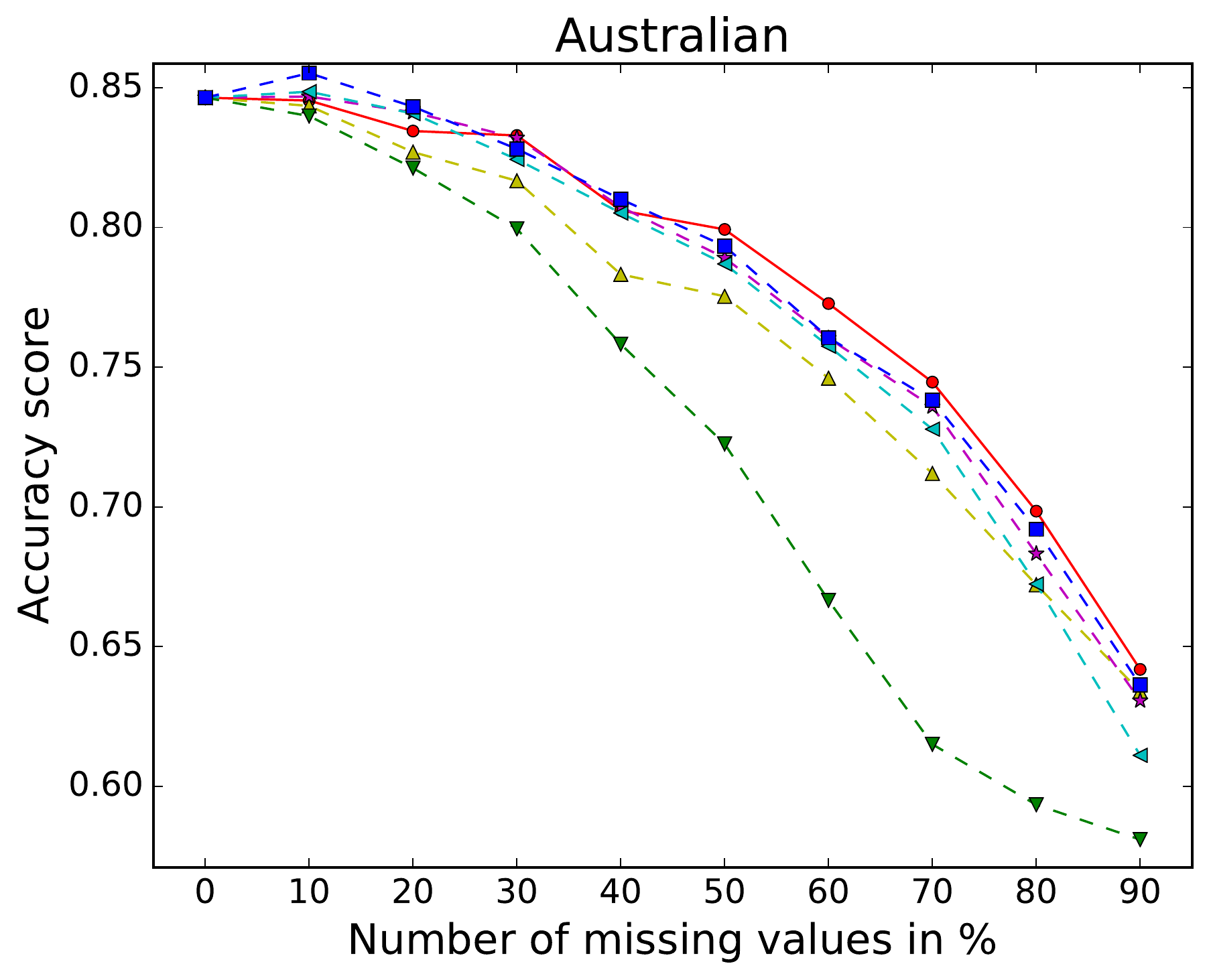}
	\includegraphics[width=0.24\textwidth]{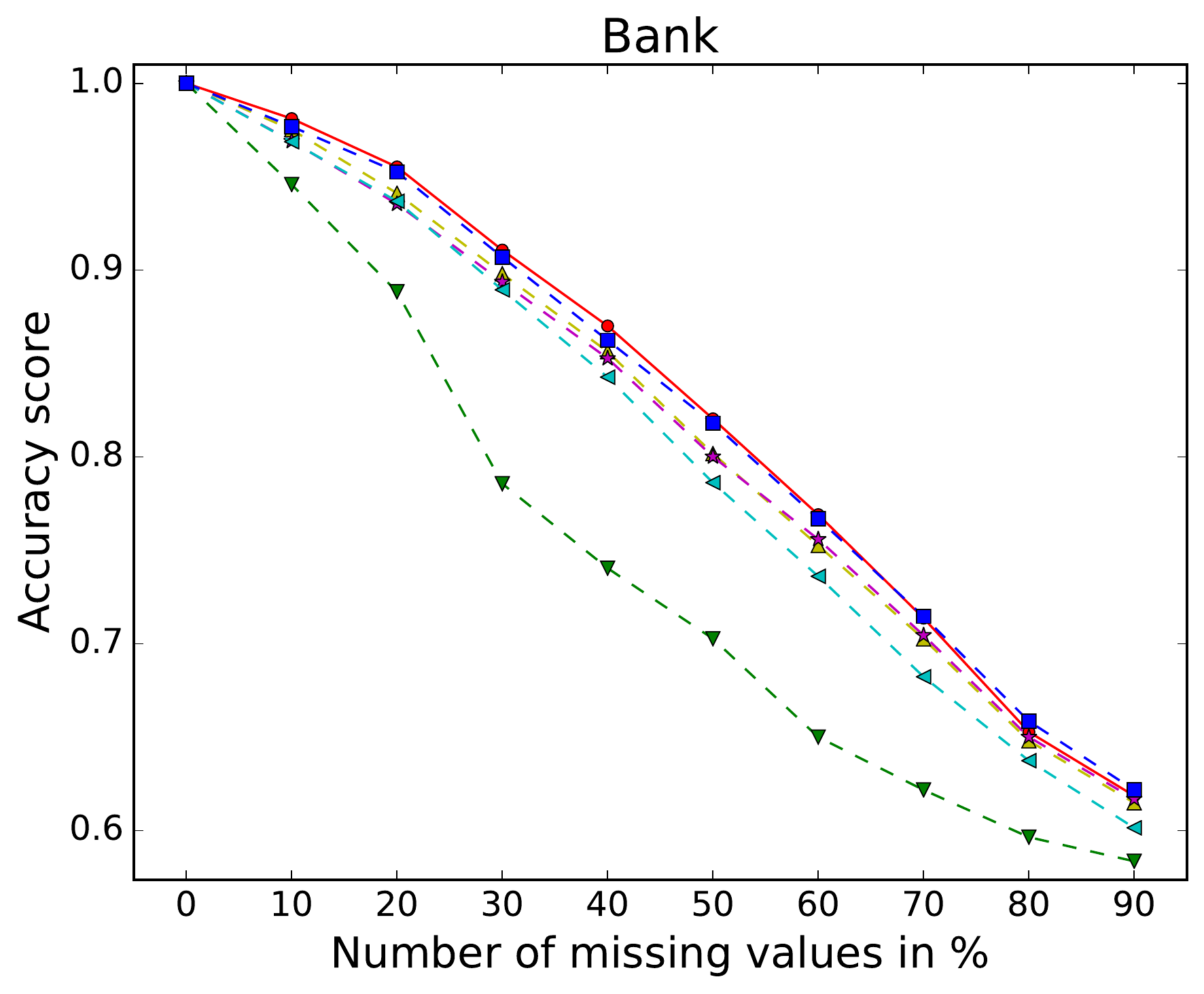}
	\includegraphics[width=0.24\textwidth]{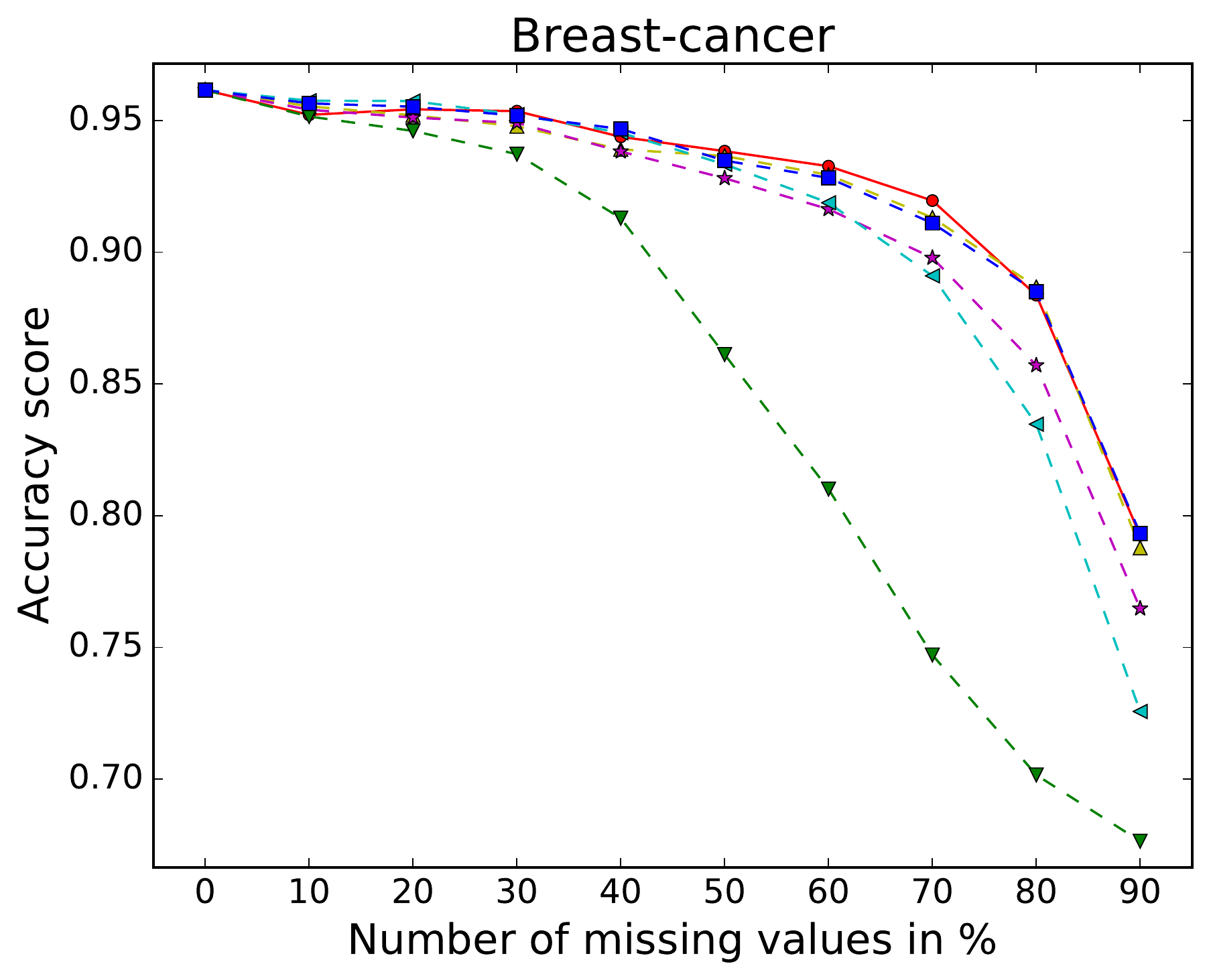}
	\includegraphics[width=0.24\textwidth]{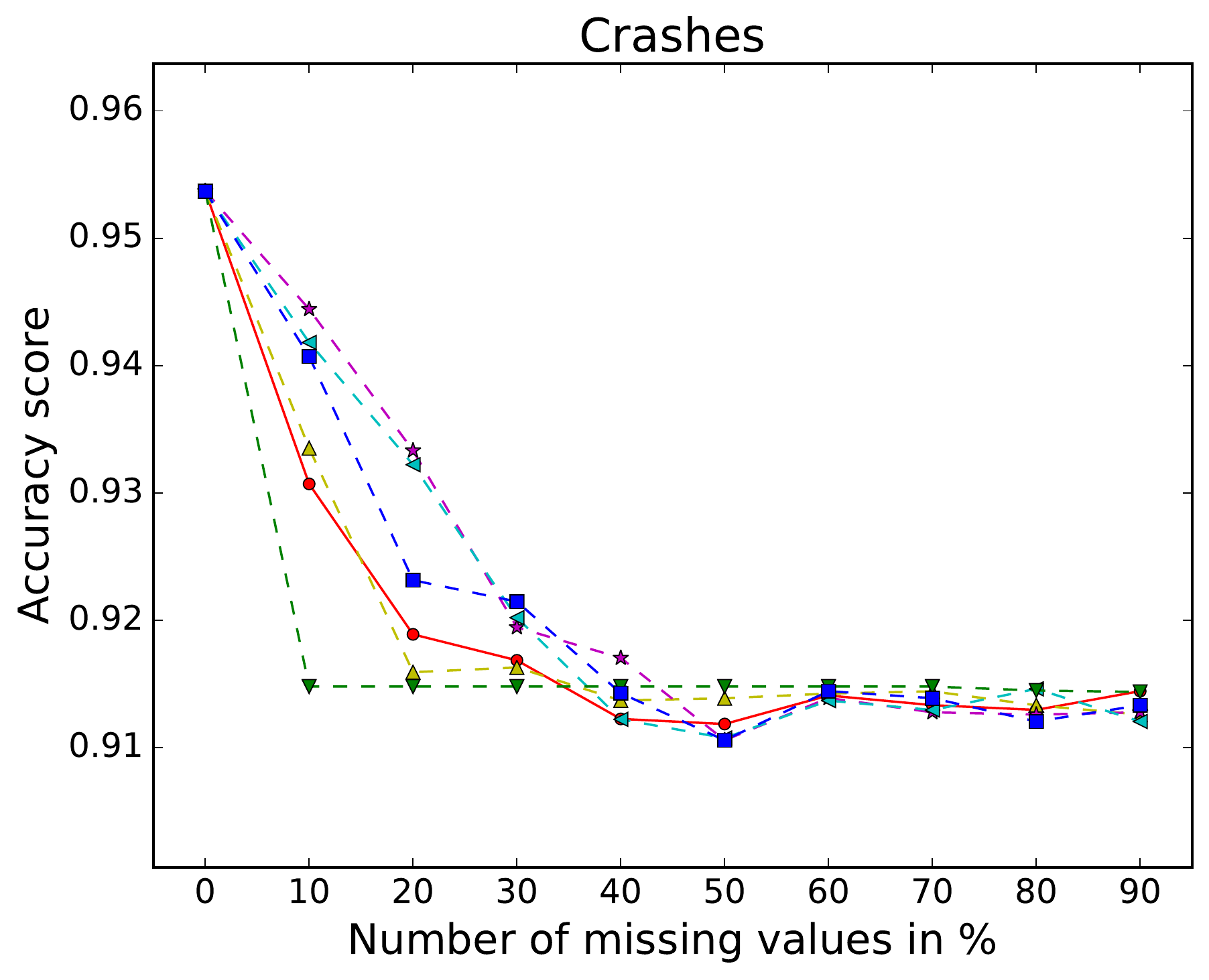}
	\includegraphics[width=0.24\textwidth]{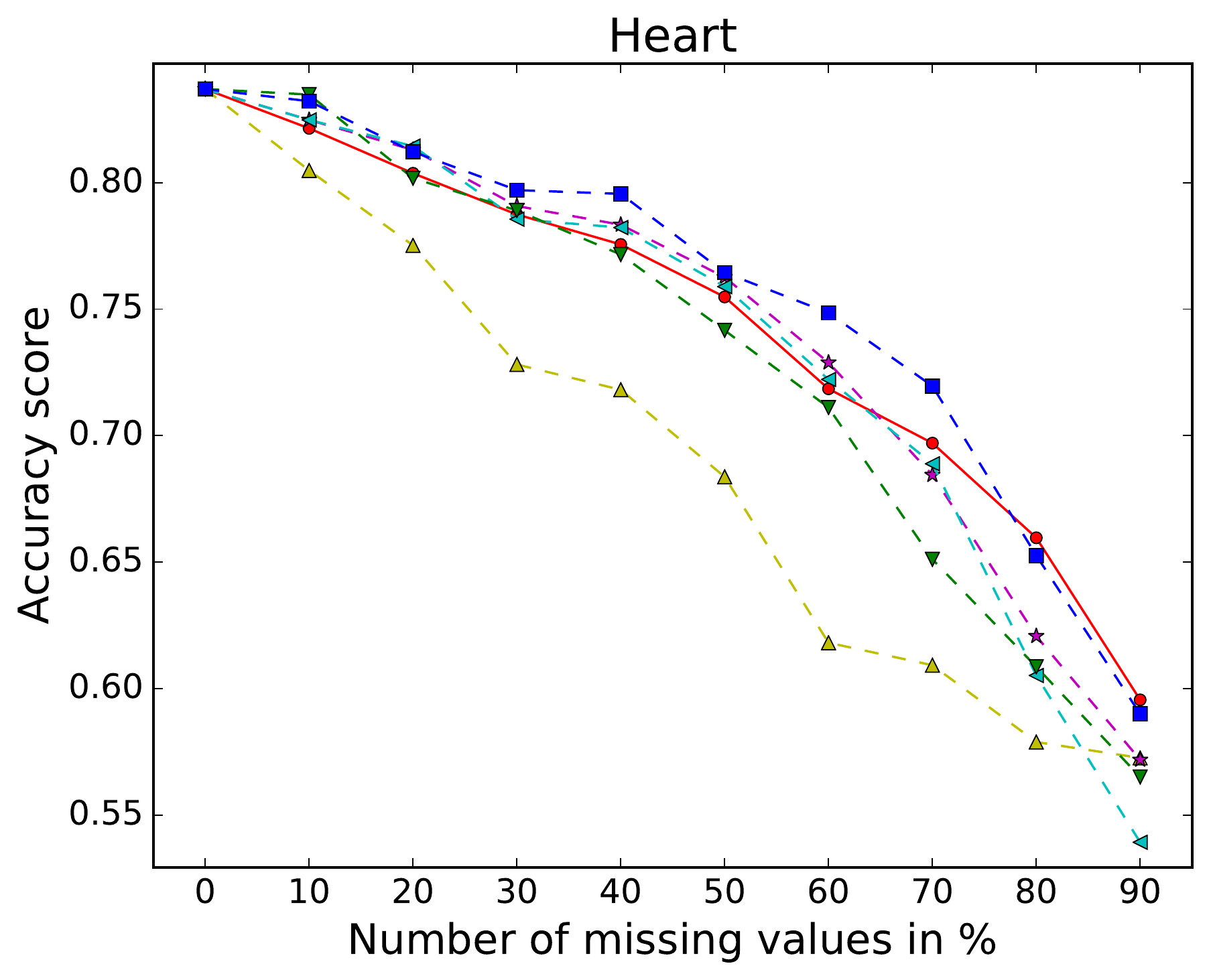}
	\includegraphics[width=0.24\textwidth]{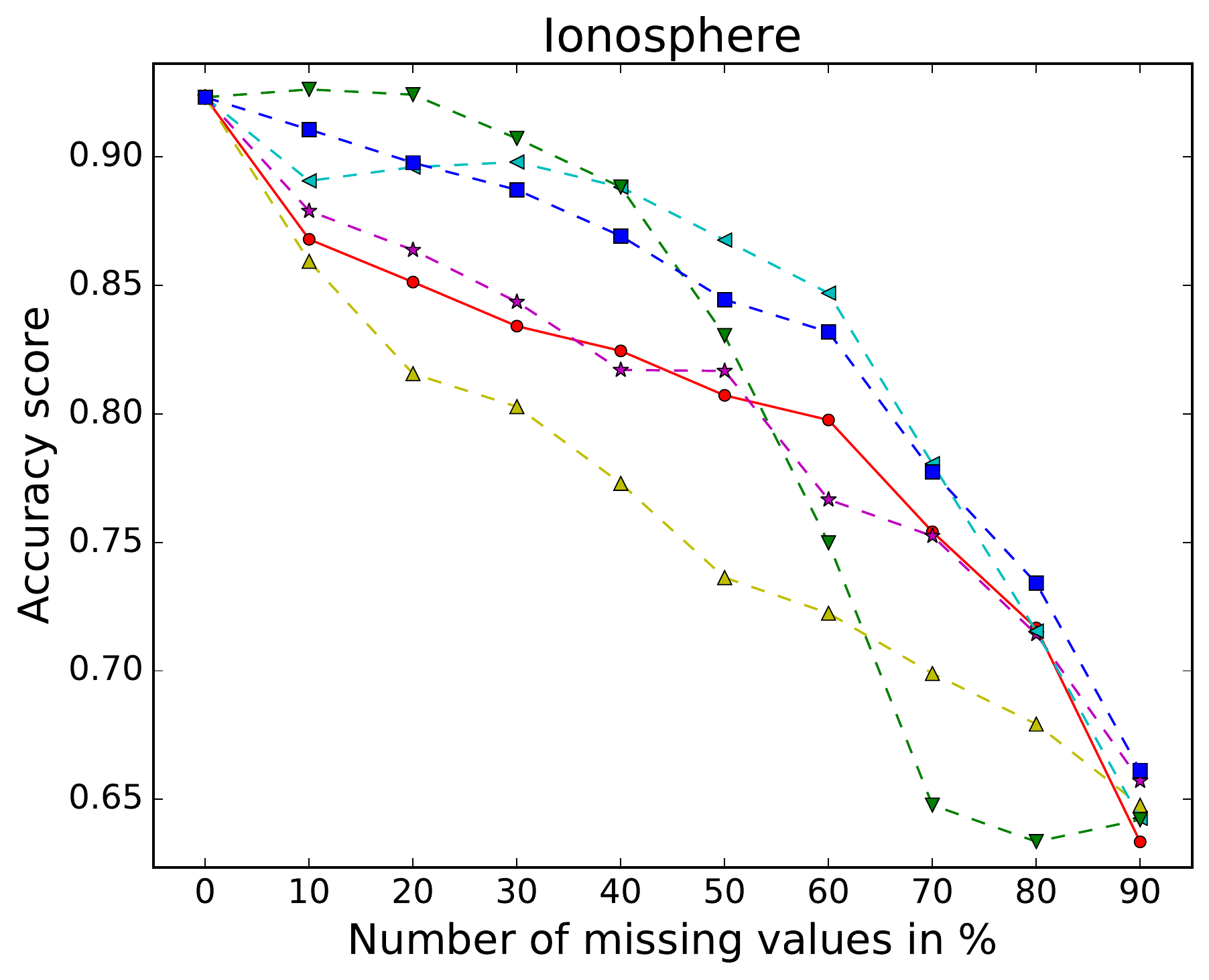}
	\includegraphics[width=0.24\textwidth]{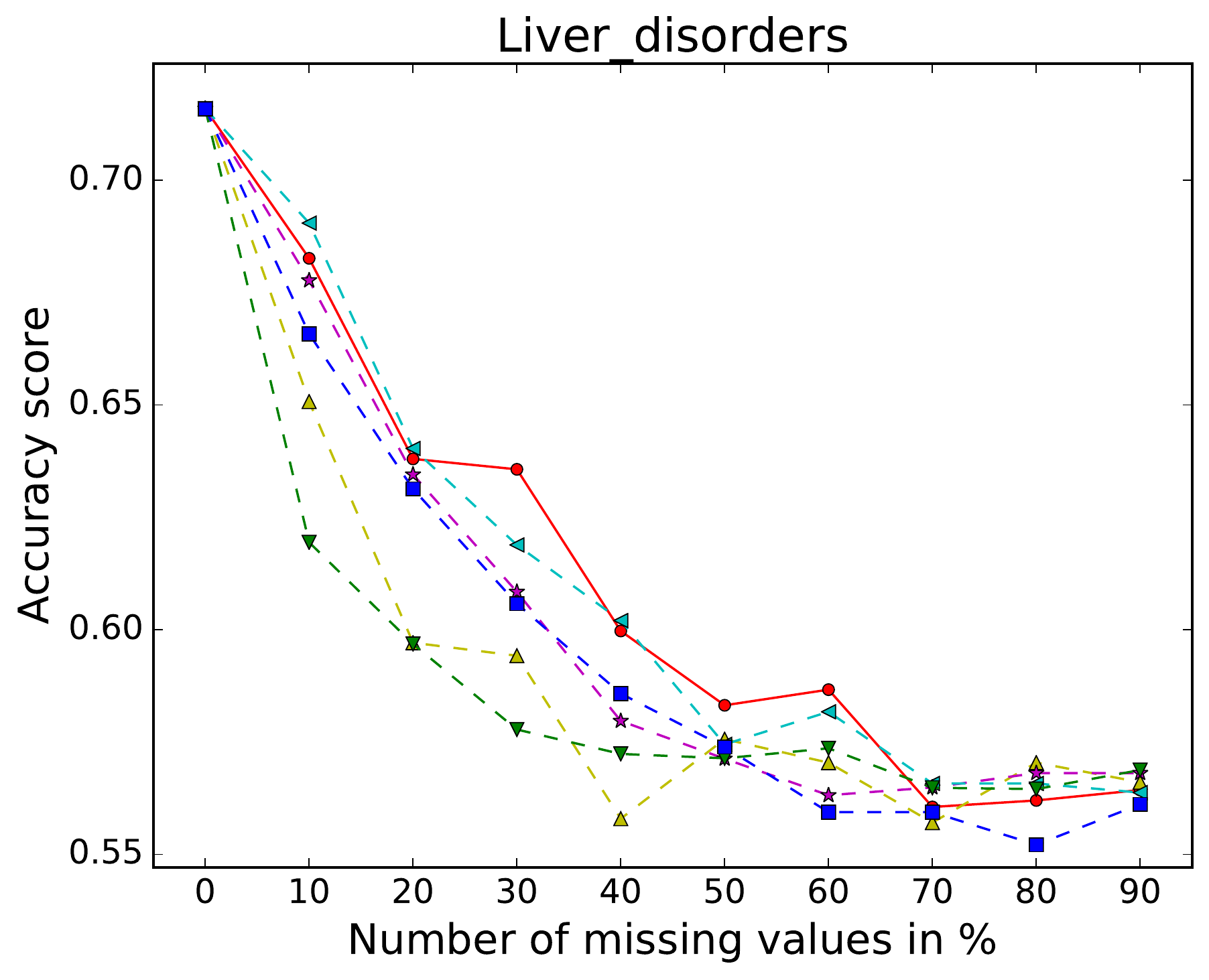}
	\includegraphics[width=0.24\textwidth]{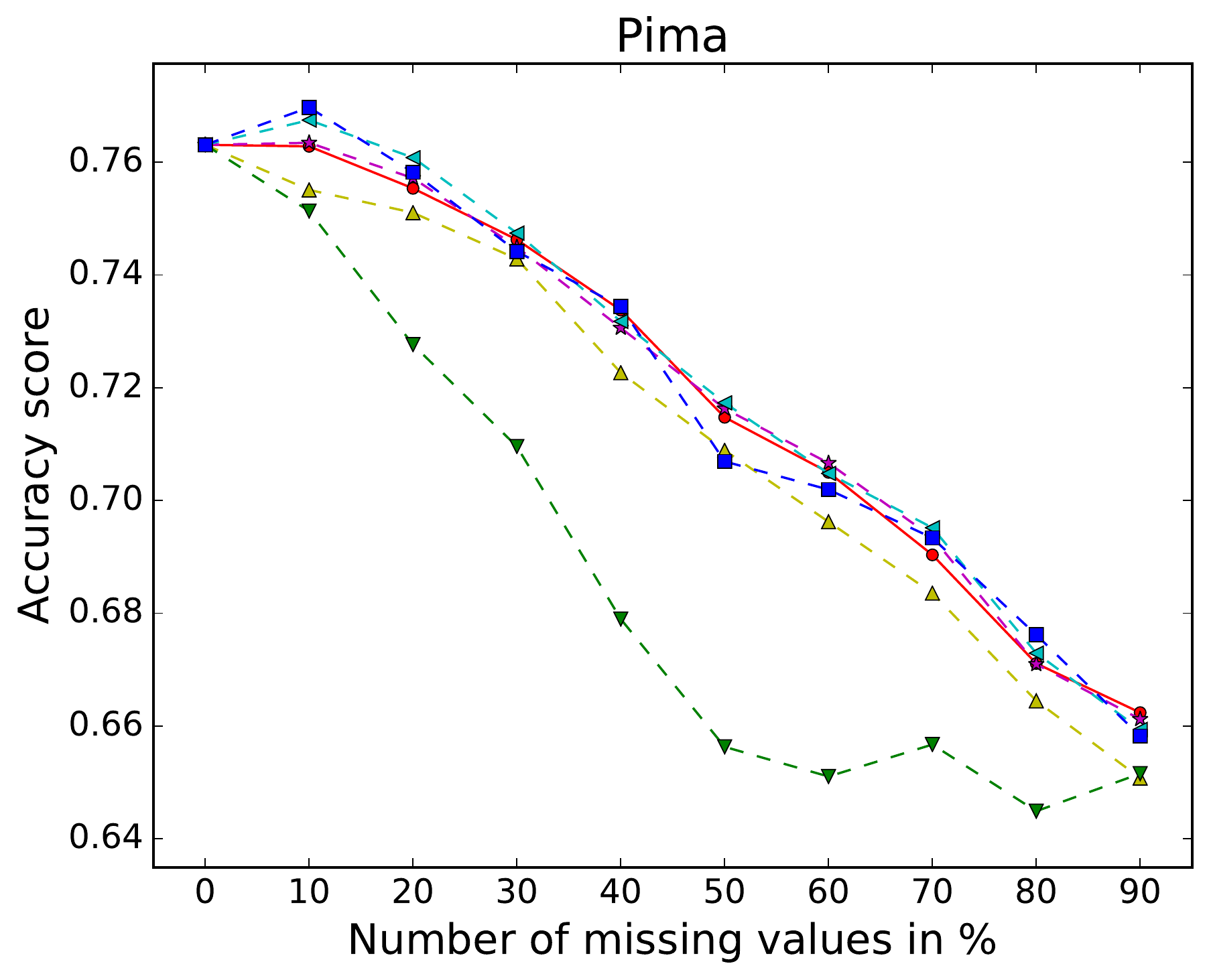}
	\caption{Classification results measured by accuracy reported on test set when missing entries satisfy MAR.}
	\label{fig:uciMAR}
\end{figure*}

\begin{figure*}[t]
	\centering
	\includegraphics[width=0.5\textwidth]{./figs/legend-crop.pdf}\\
	\includegraphics[width=0.24\textwidth]{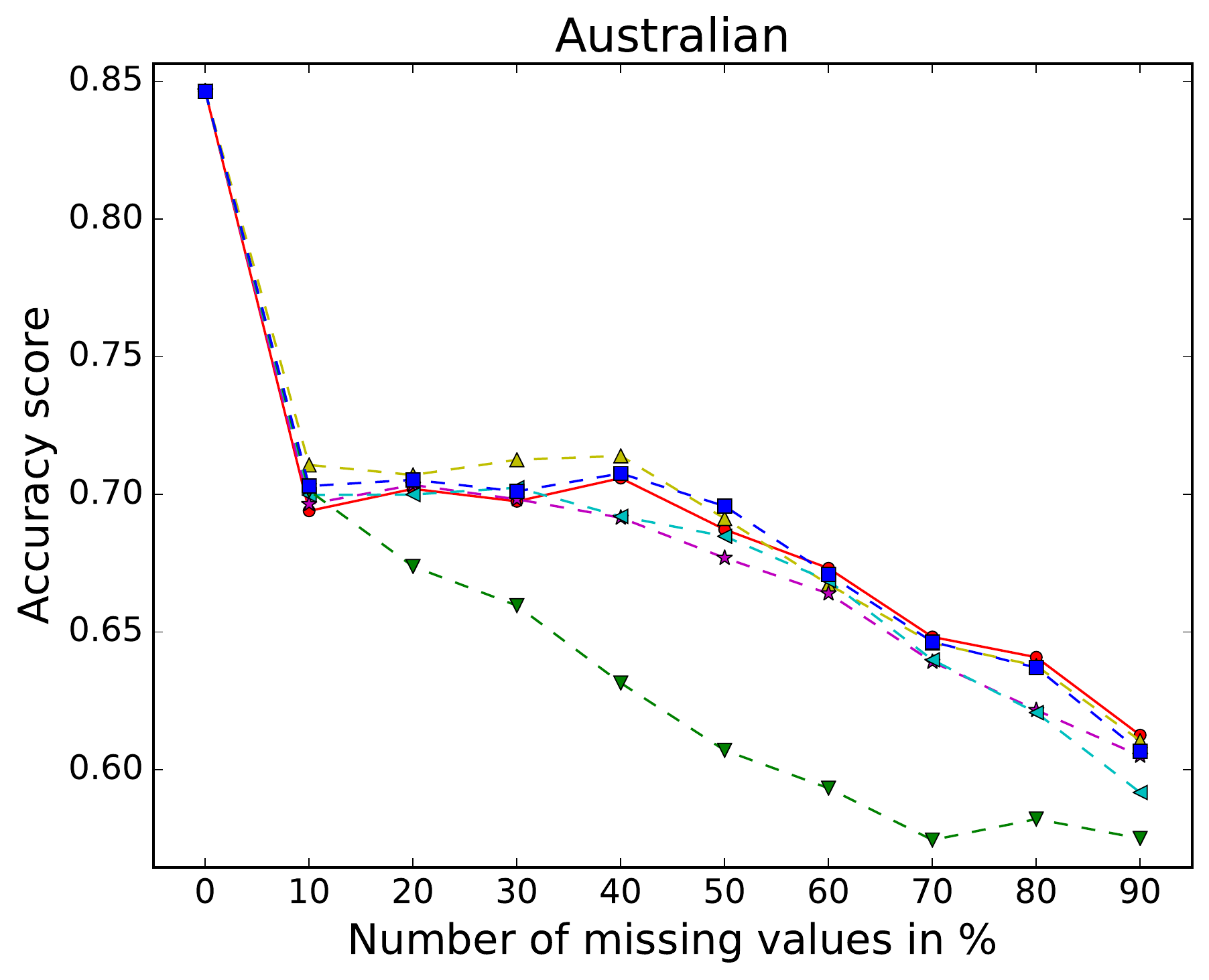}
	\includegraphics[width=0.24\textwidth]{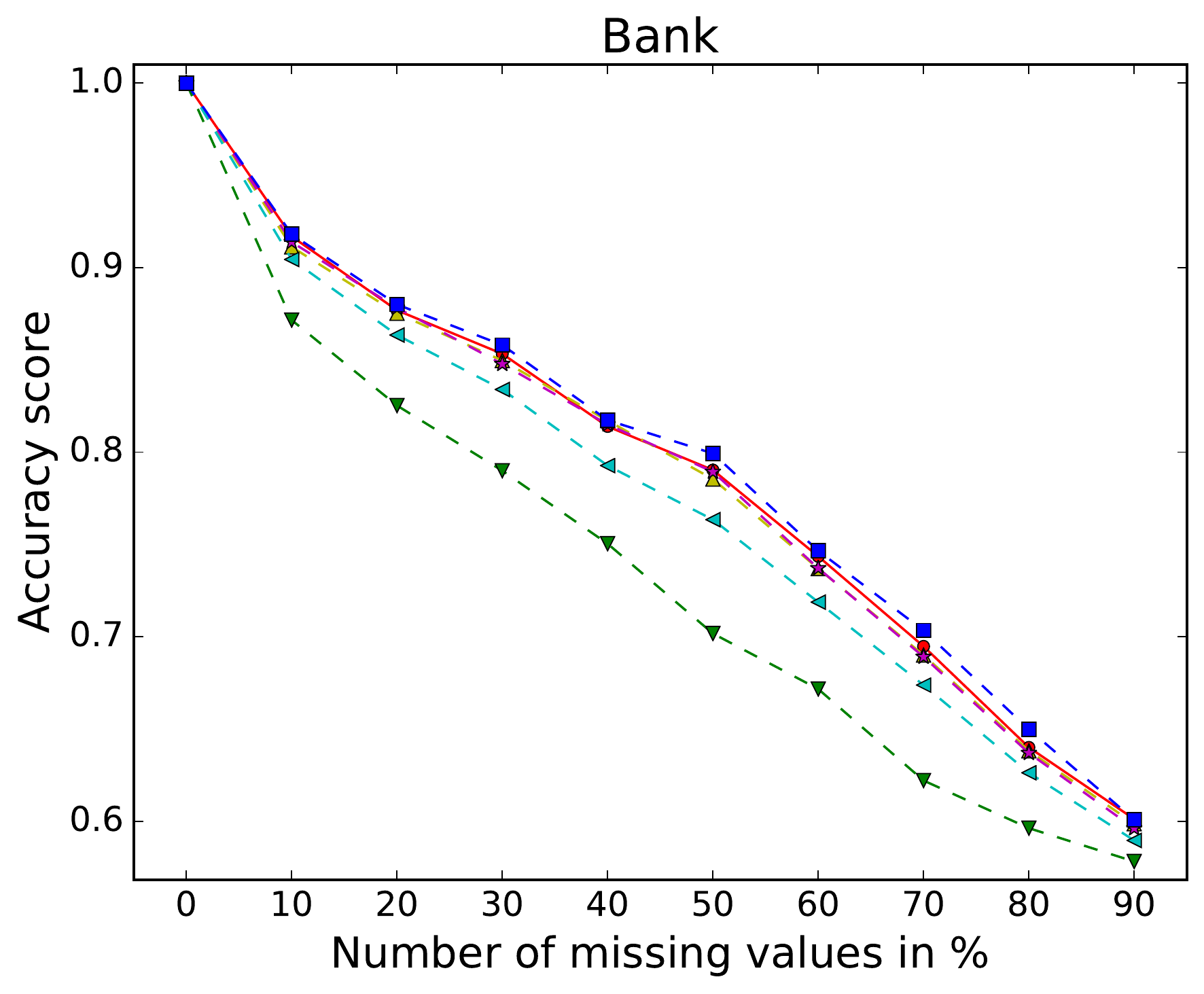}
	\includegraphics[width=0.24\textwidth]{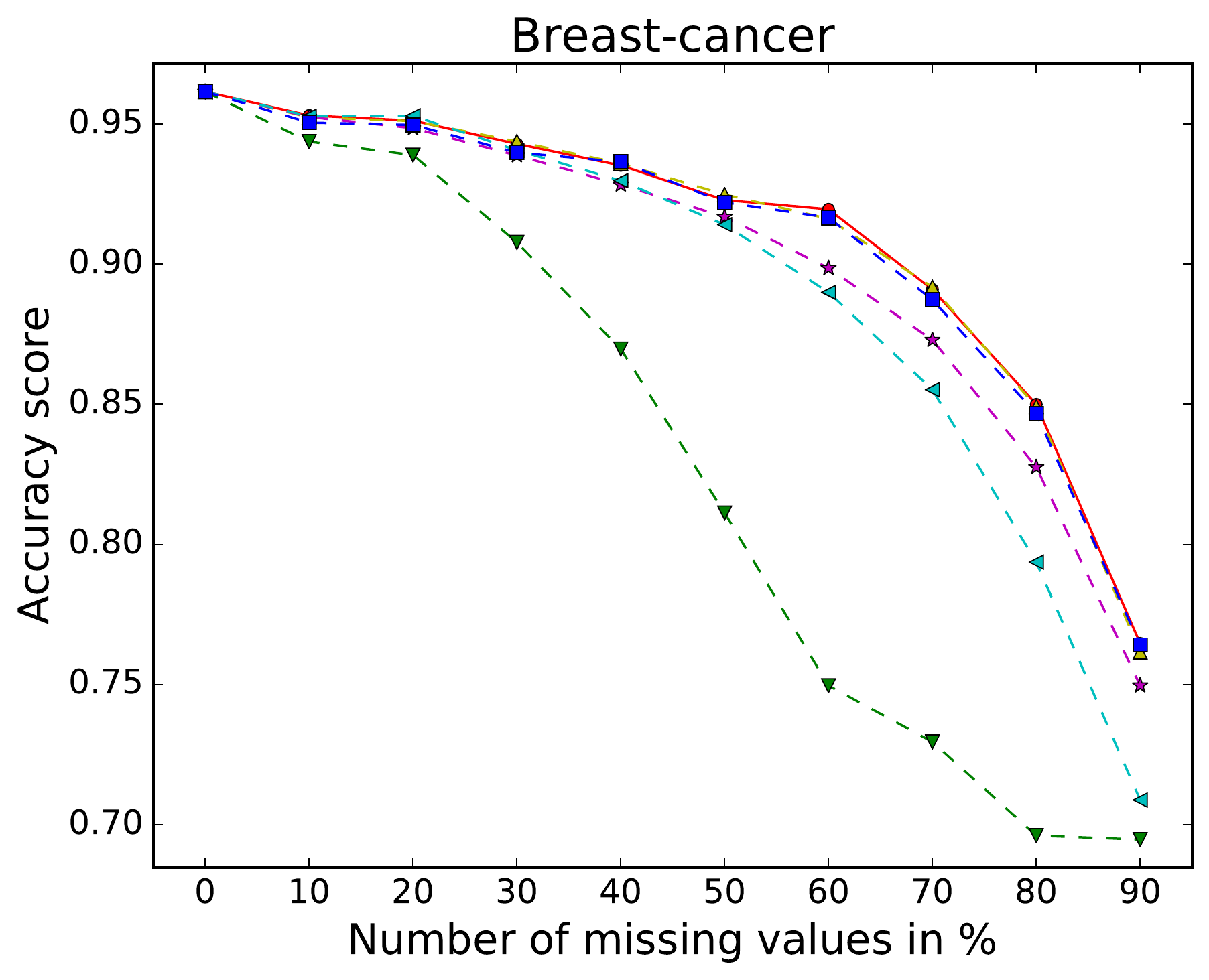}
	\includegraphics[width=0.24\textwidth]{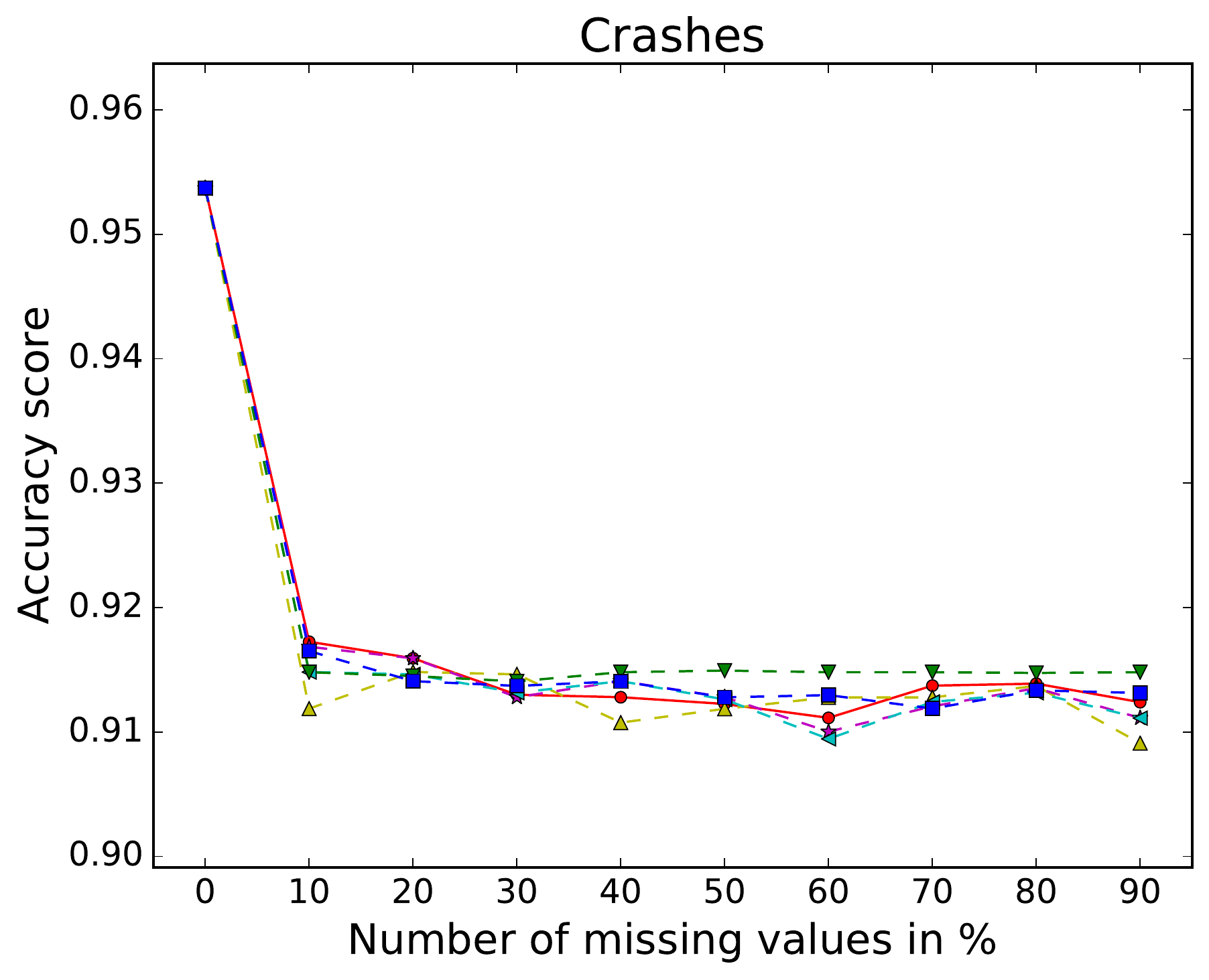}
	\includegraphics[width=0.24\textwidth]{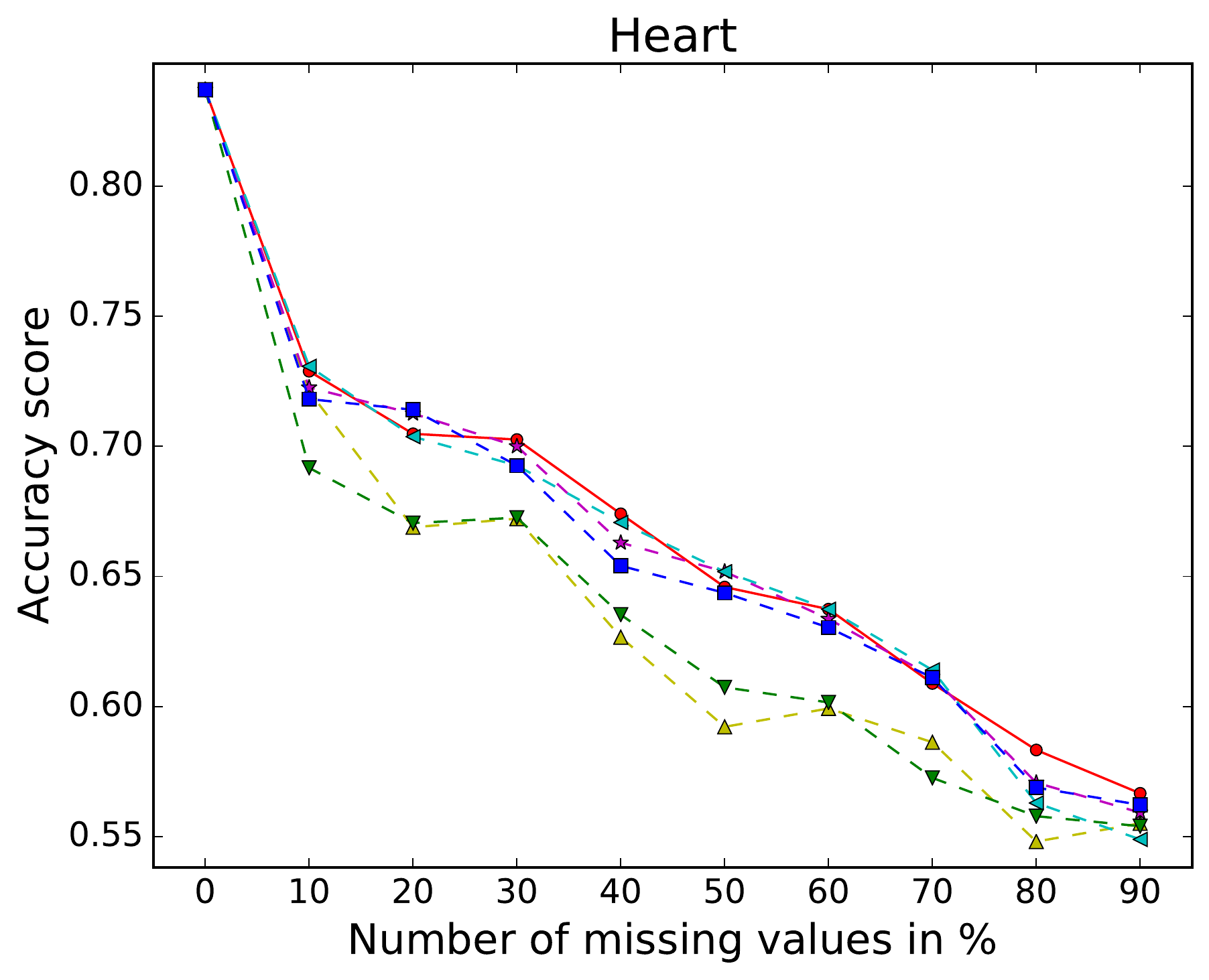}
	\includegraphics[width=0.24\textwidth]{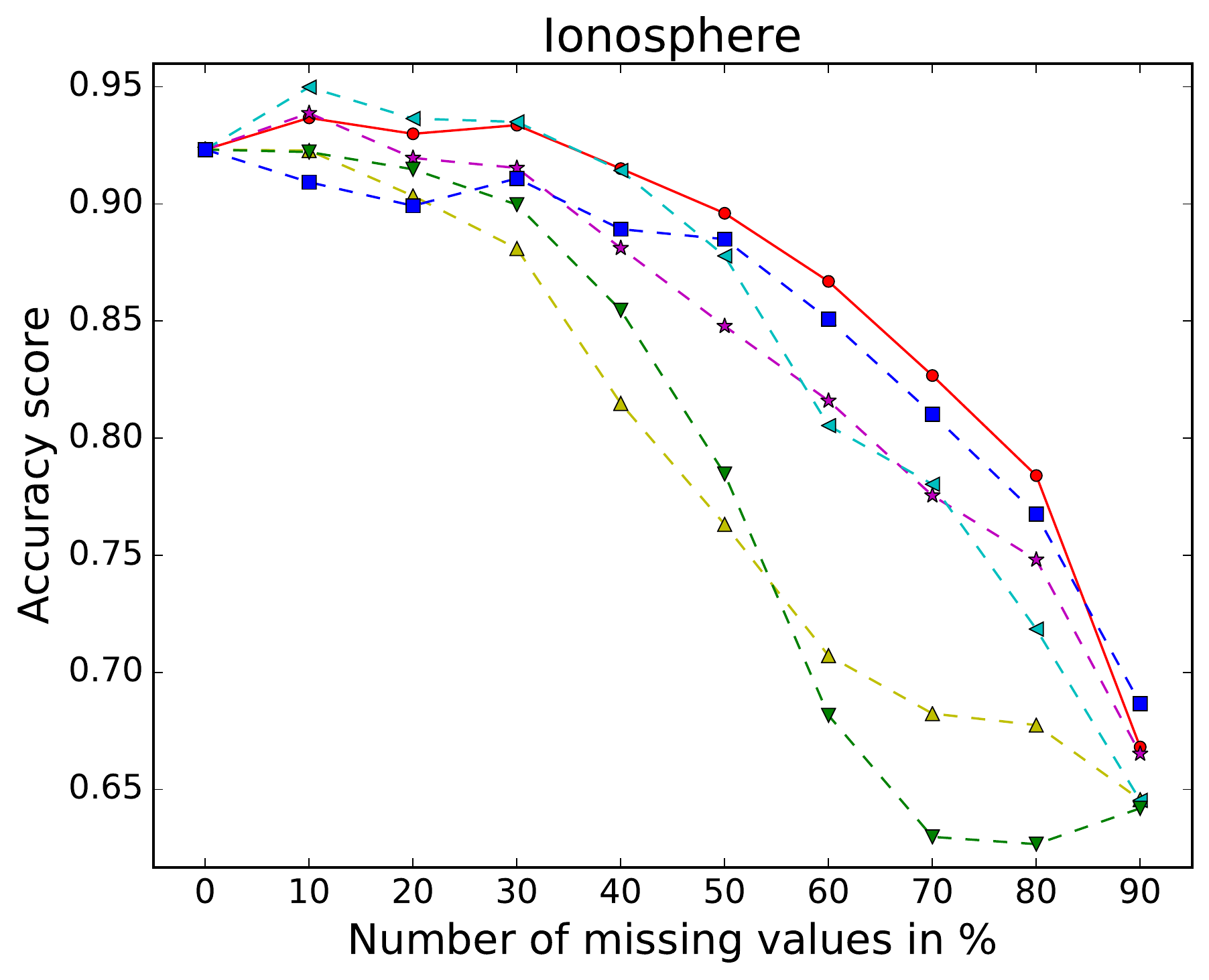}
	\includegraphics[width=0.24\textwidth]{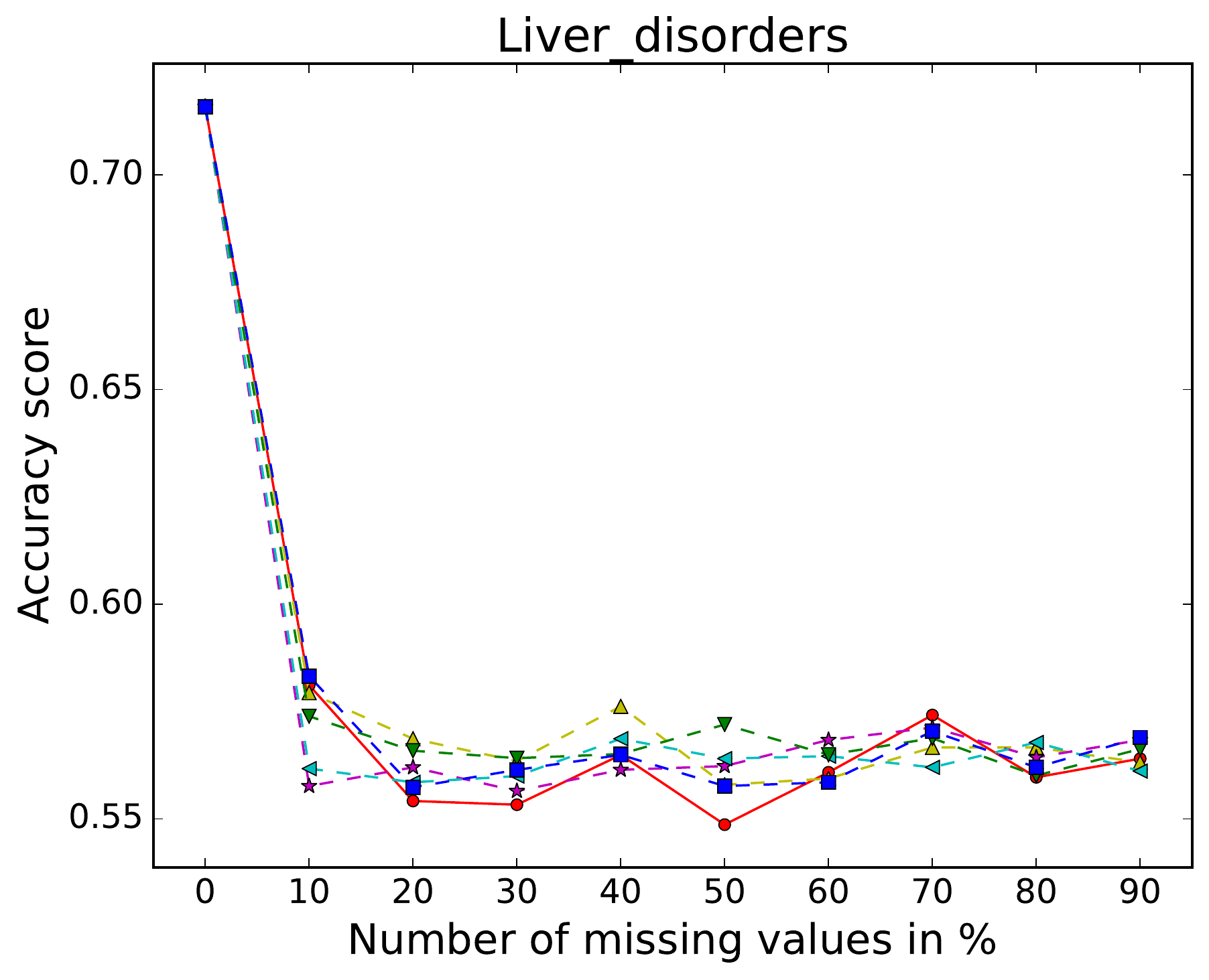}
	\includegraphics[width=0.24\textwidth]{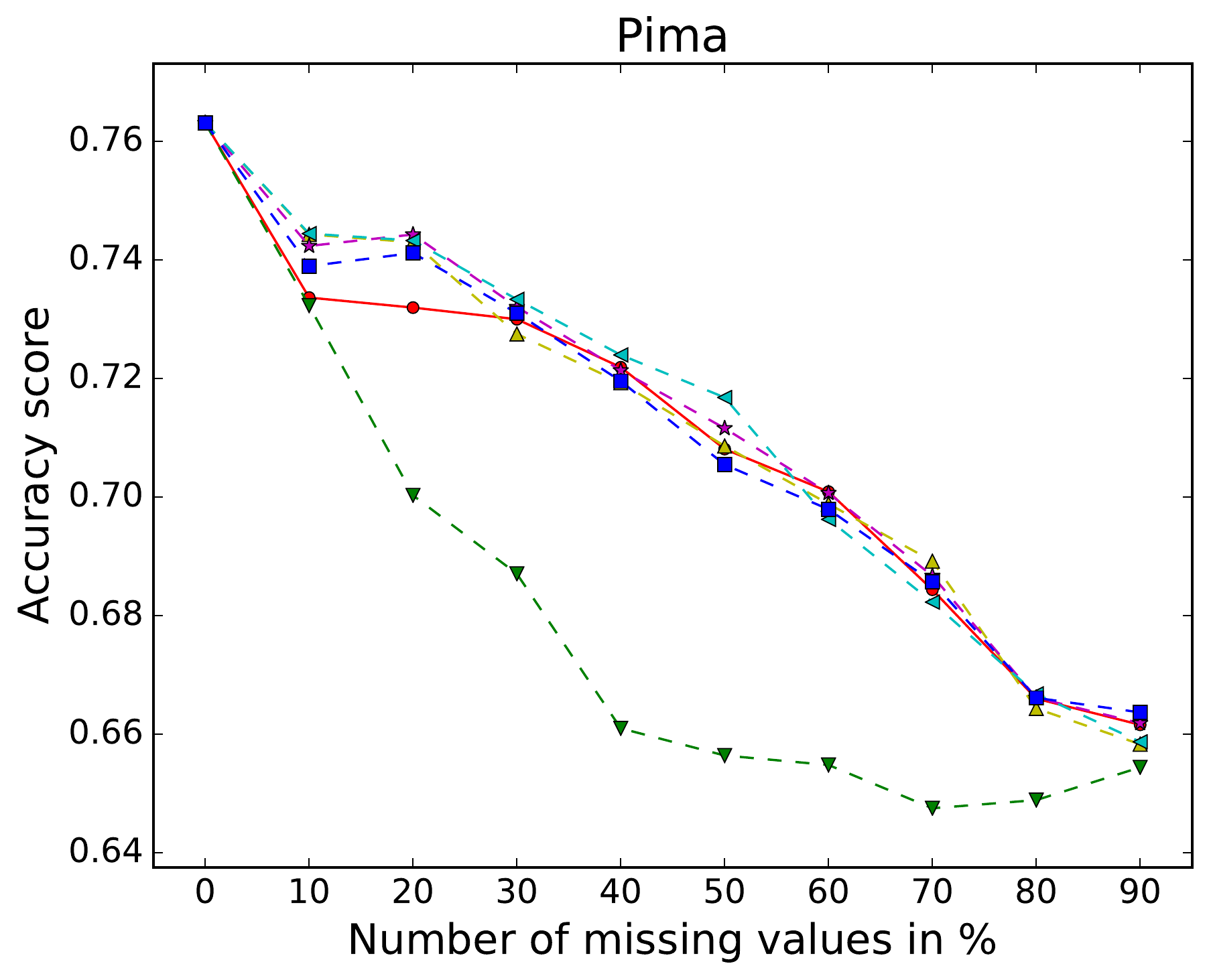}
	\caption{Classification results measured by accuracy reported on test set when missing entries satisfy NMAR.}
	\label{fig:uciMNAR}
\end{figure*}

\begin{figure}[t]
	\centering
	\includegraphics[width=0.45\textwidth]{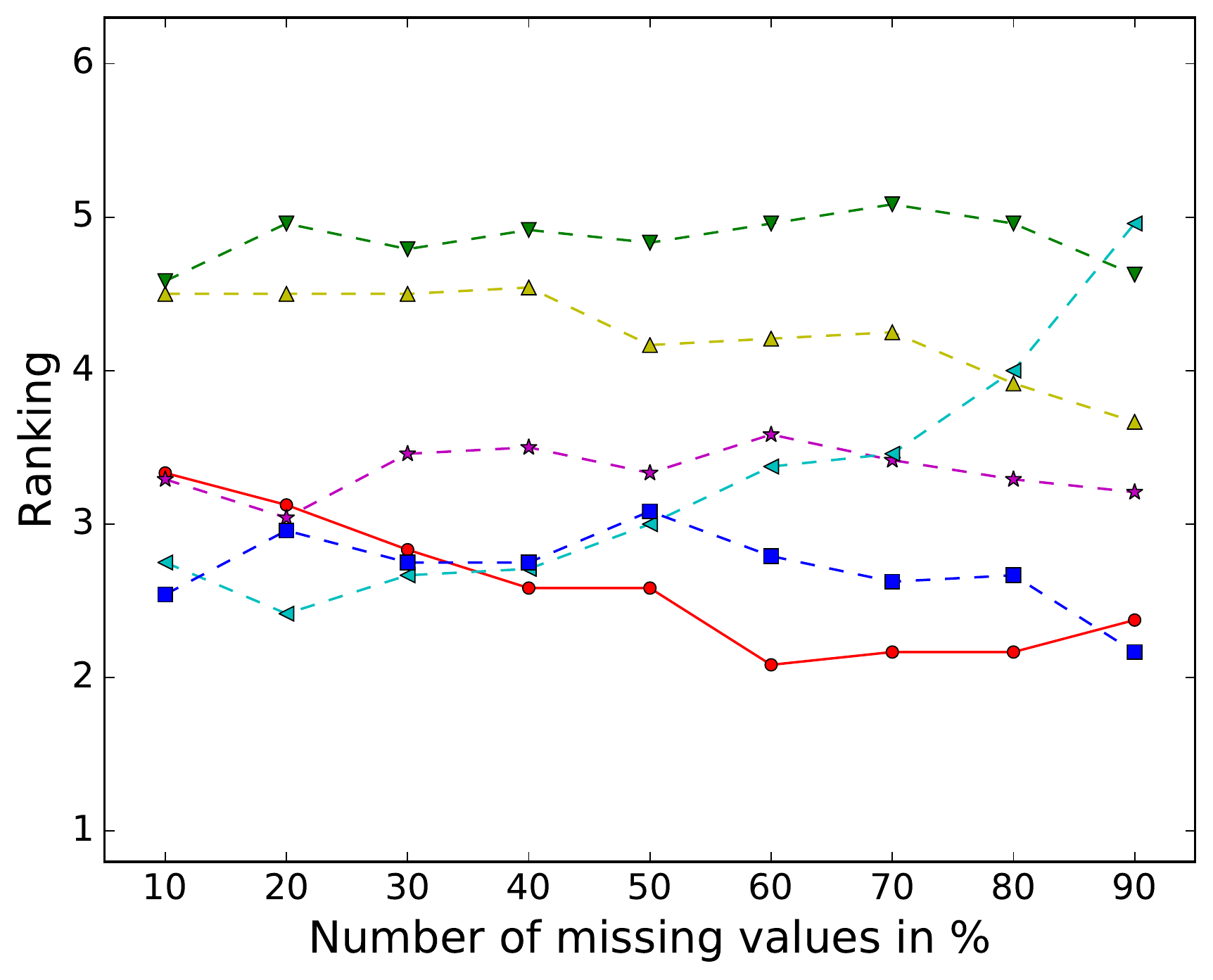}
	\caption{Ranking calculated over all data sets and missing data scenarios (the lower is the better).}
	\label{fig:rank}
\end{figure}

\begin{figure}[t]
	\centering
	\includegraphics[width=0.45\textwidth]{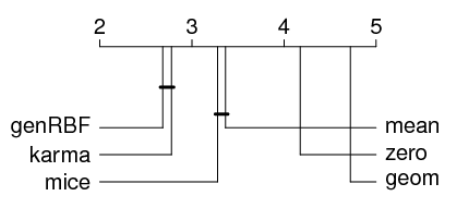}
	\caption{Visualization of statistical comparison.}
	\label{fig:stat}
\end{figure}

The following theorem gives a final formula for the \our{} kernel function.
\begin{theorem}
	Let $F=N(m,\Sigma)$ be a density on $\R^N$ and let $\sigma >0$ be fixed. We assume that $N(m^V, \Sigma^V)$ and $N(m^W, \Sigma^W)$ represent missing data points $x+V$ and $y+W$. Then, the scalar product \eqref{eq:scalarProb} equals
	\begin{equation} \label{eq:kernelTh}
	K_\sigma(x+V,y+W)=Z(V,W) \exp(-\tfrac{1}{2}\|m^V-m^W\|^2_{\hat{\Sigma}}),
\end{equation}
where $\hat{\Sigma}=2\sigma^2 I+\Sigma^V+\Sigma^W$ and the normalization factor equals:
$$
Z(V,W)= \frac{\det^{1/4}(I+\frac{1}{\sigma^2}{\Sigma}^V)\det^{1/4}(I+\frac{1}{\sigma^2}{\Sigma}^W)}{\det^{1/2}(I+\frac{1}{2\sigma^2}({\Sigma}^V+{\Sigma}^W))}.
$$
\end{theorem}
\begin{proof}
It is sufficient to apply \eqref{eq:scalarProduct} to formula \eqref{eq:scalarProb}.
\end{proof}

Let us observe that the above formula generalizes the classical RBF kernel to the case of incomplete data. Indeed, complete data points $x,y$ are represented by Dirac measures, i.e. $m^V = x, m^W = y$ and $\Sigma^V =\Sigma^W = 0$. Then $\hat{\Sigma} = 2 \sigma^2$ and
$$
K_\sigma(x,y) = \exp(-\frac{\|x-y\|^2}{4 \sigma^2}).
$$ 
Taking a parametrization $\gamma = \frac{1}{4 \sigma^2}$ we arrive at the classical formula of RBF kernel. Thus, to be consistent with a typical RBF parametrization, in the experimental section we will use the formula
$$
	K_\gamma(x+V,y+W)=Z(V,W) \exp(-\tfrac{1}{2}\|m^V-m^W\|^2_{\hat{\Sigma}}),
$$
where 
$$
\begin{array}{l}
\hat{\Sigma}=\frac{1}{2\gamma} I+\Sigma^V+\Sigma^W,\\
Z(V,W)= \displaystyle \frac{\det^{1/4}(I+4 \gamma {\Sigma}^V)\det^{1/4}(I+4 \gamma {\Sigma}^W)}{\det^{1/2}(I+2 \gamma ({\Sigma}^V+{\Sigma}^W))}.
\end{array}
$$

\section{Experiments} \label{se:experiments}

We evaluated \our{} in binary classification experiments using SVM and compared the results with methods that work on incomplete data. We used examples retrieved from UCI repository combined with different strategies for attributes removal.

\subsection{Experimental setting}

We used eight UCI datasets \cite{Asuncion+Newman:2007}, which are summarized in Table \ref{tab:uci-data}. For each one, we considered three strategies for creating missing entries, each one realizing different missing data assumption:
\begin{itemize}
\item {\bf MCAR.} We randomly removed a fixed percentage of features, $p \in \{10\%, 20\%, \ldots, 90\%\}$. 
\item {\bf MAR.} We defined a structural process for attributes removal, where the selection of missing entries were fully accounted by visible features. We drawn $N$ points $x_1,\ldots,x_N$ of a dataset $X \subset \R^N$. Then, for every $x \in X$, where $x \neq x_i$ for $i=1,\ldots,N$, we removed its $i$-th attribute with a probability 
$$
\exp(-t \|x - x_i\|_{\Sigma})),
$$
where $\Sigma$ is a sample covariance matrix taken from data and $t>0$ is fixed. In other words, $i$-th point determined the removal of $i$-th feature. The value of $t$ was fixed so that to remove approximately $10\%, 20\%, \ldots, 90\%$. 
\item {\bf NMAR.} We modified previous scenario in the following way. The set of features was randomly divided into two equally-sized parts: visible features $I_V$ and hidden features $I_H$. Given $N$ randomly selected points $x_1,\ldots,x_N$ of $X$, we removed attribute $i \in I_V$ of $x \in X$ with a probability 
$$
\exp(-t\|x^{I_H} - x_i^{I_H}\|_{\Sigma})),
$$ 
where $x^{I_H}$ denotes the restriction of $x$ to coordinates from $I_H$ (as before $t > 0$ controlled the number of removed features). After that, data were represented only by features from $I_V$, while coordinates included in $I_H$ were discarded. In other words, attributes $I_H$ were used to define a removal process, which depends on unobservable features.
\end{itemize}
The missing entries appeared in both train and test sets.

For a comparison, we used two imputation techniques as baseline, multiple imputation strategy and two state-of-the-art methods developed for SVM:
\begin{enumerate}
\item  {\bf mean}: Missing coordinates were filled with average values taken over training set.
\item {\bf zero}: Absent attributes were set to zeros. 
\item  {\bf mice}: Unknown features were filled based on a train set using Multiple Imputation by Chained Equation \cite{azur2011multiple} implemented in R package mice\footnote{\url{https://cran.r-project.org/web/packages/mice/index.html}} \cite{buuren2011micemice}, where several imputations are drawing from the conditional distribution of data by Markov chain Monte Carlo techniques. 
\item {\bf geom}: Geometric margin is a modified SVM classifier proposed by G. Chechik et. al. \cite{chechik2008max}, where no assumption about missing data mechanism is required. In this approach, an objective function is based on the geometric interpretation of the margin and aims to maximize the margin of each sample in its own relevant subspace.
\item {\bf karma}: It is an algorithm for kernel classification proposed by E. Hazan et. al. \cite{hazan2015classification}, where the linear classifier is iteratively tuned.
\end{enumerate}
To estimate a Gaussian density from incomplete data used in {\our}, we applied R package {\bf norm}\footnote{\url{https://cran.r-project.org/web/packages/norm/index.html}} on train set only (estimation stage did not have the access to test/validation set). 

Each method was combined with SVM classifier using RBF kernel and tested in double 5-fold cross validation procedure. That is, for every division into train and test sets, the required hyperparameters were tuned using inner 5-fold cross validation applied on train set. The combination of parameters maximizing mean accuracy score (on validation set) was used to learn a final classifier on a entire train set, while the performance was evaluated on a test set that was not used during training. The accuracy was averaged over all 5 trails. Additionally, to reduce the effect from random deletion of attributes, we generated 10 different samples of incomplete data and averaged final accuracy scores.

After normalization of data, a grid search was applied to find optimal values of hyperparameters. We inspected the following ranges for margin parameter $C \in \{2^k: k = -5,-3,\ldots,9\}$ and kernel radius $\gamma \in \{2^k: k=-5,-3,\ldots,15\}$. Since {\bf karma} loss is additionally parametrized by a parameter $\gamma_{karma}$, we considered\footnote{Such a small range was chosen because of relatively high computational complexity of the algorithm.} $\gamma_{karma} \in \{1,2\}$.

\subsection{Results}

First of all, we noted that the difference between the results in MCAR and MAR scenarios is slight, which might follow from the fact that in both cases the removal process was based on visible features, see Figures \ref{fig:uciMCAR} and \ref{fig:uciMAR}. This behavior changed in NMAR situation, Figure \ref{fig:uciMNAR}, where on one hand removing mechanism was more complex, but on the other hand data were represented by lower number of features (half of features were hidden). In consequence, all methods obtained worse prediction rate. In particular, for Liver disorders and Crashes no method was able to produce useful results when at least 20\% of attributes were missing (accuracy coincides with the classes ratio).

Visual inspection of the Figures suggests that {\our}, {\bf karma} and {\bf mice} gave similar results and were in general better than the other methods. It is not surprising that multiple imputation strategy performs better than simpler techniques, like zero or mean imputations. The same holds for {\bf karma} algorithm, which was recently claimed to obtain state-of-the-art performance. Low quality results produced by {\bf geom} algorithm might follow from the fact that this method ignores missing attributes and is only based on observed features, which could be beneficial for very complex removal processes.

To further analyze the results, we ranked the methods over all data sets and all missing data scenarios; the best performing algorithm got the rank of 1, the second best rank 2 etc. The results presented in Figure \ref{fig:rank} show that {\our} is best suited to the case when a lot of features are absent. Although the performance of {\bf mice} is better for 10-30\% of missing attributes, its results deteriorate heavily as the number of missing entries increases. It is worth to notice that the rank of {\bf karma} is very stable. Almost always, it was the second best approach.

We also verified the results applying statistical tests, see \cite{demvsar2006statistical}, specifically we used the Friedman test with Nemenyi post hoc analysis. Given a ranking of the methods (aggregated additionally over percentage of missing coordinates), the analysis consists of two steps: 
\begin{itemize}
\item the null hypothesis is made that all methods perform the same and the observed differences are merely random (the hypothesis is tested by the Friedman test, which follows a $\chi^2$ distribution)
\item having rejected the null hypothesis the differences in ranks are analyzed by the Nemenyi test.
\end{itemize}

Figure \ref{fig:stat} visualizes the results for a significance level of $p=0.05$. The x-axis shows the mean rank over combinations of data set and percentage of missing values for each method. Groups of methods for which the difference in mean rank is not statistically significant are connected by horizontal bars. As can be observed, the mean rank of {\our} is better than the others. This advantage is statistically significant comparing it with all methods except {\bf karma}. Nevertheless, the use of our method is much simpler, since it relies on applying classical SVM with slightly modified kernel function, whereas {\bf karma} uses an iterative algorithm to increase the performance of the classifier and requires the selection of one more hyperparameter.

\section{Conclusion}

We proposed \our{}, the generalization of RBF kernel to the case of incomplete data. This method uses the information contained in data distribution to model the uncertainty on absent attributes without performing any direct imputations. The experimental results show that \our{} outperforms imputation-based techniques and obtains slightly better results than recent state-of-the-art algorithm. Moreover, it does not require the modification of existing machine learning methods, which makes it easy to use in practice.

\newpage

\bibliography{./ref}
\bibliographystyle{icml2017}

\end{document}